\documentclass[accepted]{uai2022_supplement} 

\usepackage[american]{babel}

\usepackage{natbib} 

\usepackage{mathtools} 
\usepackage{booktabs} 
\usepackage{tikz} 

\usepackage{algorithm}
\usepackage{algorithmicx}
\usepackage[noend]{algpseudocode}

\usepackage[american]{babel}

\usepackage{amsthm}

\usepackage{graphicx}
\usepackage{amsmath, amssymb} 
\usepackage{float}
\usepackage{xcolor}
\usepackage{tikz}

\DeclareMathOperator{\Tr}{Tr}

\newcommand{\x}{\mathbf{x}}

\newcommand{\z}{\mathbf{z}}
\newcommand{\y}{\mathbf{y}}

\newcommand{\uu}{\mathbf{u}}

\newcommand{\vb}{\mathbf{v}}

\newcommand{\f}{\mathbf{f}}
\newcommand{\g}{\mathbf{g}}
\newcommand{\R}{\mathbb{R}}

\newcommand{\C}{\mathbf{C}}
\newcommand{\I}{\mathbf{I}}

\newcommand{\A}{\mathbf{A}}

\newcommand{\W}{\mathbf{W}}

\newcommand{\q}{\mathbf{q}}

\newcommand{\mub}{\boldsymbol{\mu}}
\newcommand{\Sigmab}{\boldsymbol{\Sigma}}
\newcommand{\Qb}{\mathbf{Q}}
\newcommand{\n}{\mathbf{n}}
\newcommand{\0}{\mathbf{0}}

\renewcommand{\texttt}[1]{#1}

\newtheorem{theorem}{Theorem}
\newtheorem{corollary}{Corollary}

\title{Binary Independent Component Analysis: A  Non-stationarity-based Approach}

\author[1,2]{\href{mailto:<antti.hyttinen@helsinki.fi>?Subject=Your UAI 2022 paper}{Antti Hyttinen}{}}
\author[1,2,3]{\href{mailto:<vitoria.barin-pacela@mila.quebec>?Subject=Your UAI 2022 paper}{Vitória Barin-Pacela}{}}
\author[1]{\href{mailto:<aapo.hyvarinen@helsinki.fi>?Subject=Your UAI 2022 paper}Aapo~Hyv\"arinen}
\affil[1]{%
    Department of Computer Science\\
University of Helsinki\\
   Helsinki, Finland
}
\affil[2]{%
    Helsinki Institute for Information Technology, Finland
}
\affil[3]{%
    Mila\\
    Universit\'e de Montr\'eal\\
    Montr\'eal, Canada
  }

\begin{document}
\maketitle

\begin{abstract}
We consider independent component analysis of binary data. While fundamental in practice, this case has been much less developed than ICA for continuous data. We start by assuming a linear mixing model in a continuous-valued latent space, followed by a binary observation model. Importantly, 
we assume that the sources are non-stationary; 
this is necessary since any non-Gaussianity would essentially be destroyed by the binarization.
Interestingly, the model allows for closed-form likelihood by employing the cumulative distribution function of the multivariate Gaussian distribution. In stark contrast to the continuous-valued case, we prove non-identifiability of the model with few observed variables; our empirical results imply identifiability when the number of observed variables is higher. We present a practical method for binary ICA that uses only pairwise marginals, which are faster to compute than the full multivariate likelihood. 
Experiments give insight into the requirements for the number of observed variables, segments, and latent sources that allow the model to be estimated.
\end{abstract}

\section{Introduction}

Despite significant progress in both linear and nonlinear ICA in recent years~\citep{tcl, hyvarinen19, Khemakhem2019}, ICA for binary data remains a challenging and important problem as binary data is abundant in various fields, such as bioinformatics, health informatics, social sciences, natural language, and electrical engineering. An ICA model for binary data may also open new opportunities in solving problems closely related to ICA, such as
causal discovery~\citep{Shimizu06JMLR} and feature extraction~\citep{tcl}.

Methods for binary ICA have been proposed based on either binary or continuous-valued independent components. In the case of binary components, \citet{Himberg01} and \citet{nguyen10} assumed an OR mixture model.  
In addition, some extensions of Latent Dirichlet can be seen as binary ICA \citep{Podosinnikova15, Buntine05}.
On the other hand, \cite{kaban06} presented an approach based on a latent linear model and binarized observations, although the components were restricted to the unit interval, which limits its applicability.
Recently, \citet{Khemakhem2019} presented a nonlinear ICA model (iVAE) that can employ binarized observations, making several contributions that we can build on.

Our goal is to study the prospects of ICA for binary data using a model that is both theoretically analyzable and intuitively appealing. 
It is crucial to investigate the identifiability of such a model, and to have a consistent estimator which is not based on approximations whose validity are not clear.
None of the approaches above fulfills all of these criteria.\footnote{As noted in the Corrigendum of \citet{Khemakhem2019} (v4 on arXiv), their initial identifiability proof for a discrete non-linear ICA model is incorrect.}

We propose a binary ICA model  
inspired by recent developments in nonlinear ICA.
We formulate a latent linear model with a separate binarizing measurement equation. 
Crucially, we assume the components to be non-stationary, which is a powerful principle and very useful here because any non-Gaussianity (commonly employed in ICA) may be destroyed by binarization.
Thus, we obtain a binary ICA model whose likelihood can be described in closed-form
via the multivariate Gaussian cumulative distribution function. We further propose to combine the likelihood with a moment-matching approach to obtain a fast and accurate estimation algorithm. In fact, due to the model structure, pairwise marginal distributions of non-binarized data can be accurately estimated from the binary data and the likelihood can be computed directly from them. We investigate the identifiability of the model, and somewhat surprisingly, we show that low-dimensional models are in fact non-identifiable---while higher-dimensional models are (empirically) shown to be identifiable. 

\section{A MODEL FOR BINARY ICA}
\label{model}

In this section, we define a binary counterpart of the linear ICA model. In particular, we consider here a model based on non-stationarity of the components, and start by motivating such an approach.

\subsection{The approach of non-stationarity}

While often non-stationarity is considered a nuisance, in the theory of ICA it is well-known that a suitable non-stationarity of the independent components can be very useful. \citet{Pham01} already used it in the case of linear ICA, and \citet{tcl} extended the idea to nonlinear ICA. Note that the mixing is assumed stationary, and the non-stationarity is a statistical property of the components only.

In line with such literature, we assume the $n$-dimensional data is divided into $n_u$ segments which express the non-stationarity, i.e.\ the segments have different distributions. In the case of time series, we may be able to find such segmentation simply by taking time bins of equal sizes. Such non-stationarity based on a segment-wise (piece-wise stationary) model is well-known in linear ICA \citep{Pham01,JSSv076i02}. Formally, each data point has a segment index $u$ assigned to it. 

In fact, this setting is more general and it is not necessary to have time-series. The additionally "observed" variable $u$ makes the non-stationarity a special case of the auxiliary variable framework of \citet{Khemakhem2019}. 
It is thus not only natural in the case of non-stationary time series, but also when there is any other external discrete variable, such as the experimental condition or intervention, or even a class label that modulates the distribution of the data.

The motivation for such a non-stationary model is that it can greatly extend the identifiability of ICA. 
Linear ICA is identifiable if the components are simply non-Gaussian, which is why the utility of non-stationarity in that context has always been dubious and such algorithms are rarely used. However, in the case of \textit{non}linear ICA, non-Gaussianity does not enable identifiability, which may be intuitively clear since a nonlinear transformation can change the marginal distributions quite arbitrarily from non-Gaussian to Gaussian or vice versa. A major advance was in fact obtained by \citet{tcl}, who showed that non-stationarity does enable identifiability in the nonlinear case.

Here, we propose that using non-stationarity of the components is very useful in the case of binary data as well. Again, intuitively, non-Gaussianity is likely to be rather useless since the binarization destroys any detail about the non-Gaussianity of the distributions, and such a model would be unlikely to be identifiable. However, non-stationarity is \textit{not} destroyed by binarization. Thus, binary ICA can be estimated based on non-stationarity of the components, as we will show later in this paper.

\subsection{Formal model definition}

To define the model in detail, we assume the $n$-dimensional data is generated from $n_z$ latent variables (independent components, or sources), collected into a latent random vector $\mathbf{z}^u$, which are generated independently of each other from a Gaussian distribution. Crucially, the parameters of the Gaussian distribution change as a function of the segment as
$$
\mathbf{z}^u \sim \mathcal{N}(\mub_\z^{u},\Sigmab_\z^{u})
$$
where $\Sigmab^u_\z$ is a diagonal matrix of the source
variances in segment $u$.

We define ``intermediate'' variables $\mathbf{y}^u$ which are a linear mixing of the sources by a mixing matrix $\A$ with $n$ rows and $n_z$ linearly independent columns
\begin{eqnarray}
\mathbf{y}^u = \mathbf{A} \mathbf{z}^u \sim \mathcal{N}(\A \mub_\z^{u}, \ \A \Sigmab_\z^{u} \A^{\intercal}). \label{eq:mixing}
\end{eqnarray}
Here the mixing matrix $\A$ is constant, i.e., stationary, over the segments $u$ \citep{Pham01}.

While some work in ICA considers noisy continuous observations by adding noise to $\mathbf{y}^u$, we can consider here binarized observations $\mathbf{x}^u$ instead.
The binarization is done using a linking function $\sigma$ so that the probability of $i$th element of $\mathbf{x}^u$ being 1 is:
$$
P(x_i^u=1) = \sigma(y_i^u).
$$

We use a linking function based on the Gaussian CDF (cumulative distribution function):
$$
\sigma(y_i^u) = \Phi \left(\sqrt{\frac{\pi}{8}} y_i^u \big| 0,1 \right)
$$
where $\Phi$ is the cumulative distribution function of the Gaussian distribution, here with mean $0$ and variance $1$. We use $\sqrt{\pi/8}$ as the coefficient to match closely to the  sigmoid function $\sigma(y_i) = \frac{1}{1+e^{-y_i}}$~\citep{waissi,sigmoidtrick}, which is standardly used in statistics and machine learning in similar linking contexts.

We directly allow for different coefficients instead of $\sqrt{\pi/8}$, but our estimation methods assume that the linking function has the particular form. The motivation is to allow for closed-form expressions of the Gaussian integrals involved in Section~3 in terms of the Gaussian CDF. The difference to the logistic function is very small, while the methods are much simpler with the used linking function. In fact, our ICA model allows for closed-form likelihood with this particular linking function (Section~3), which would be difficult to achieve with a logistic linking function.

Furthermore, the linking function has the following intuitive interpretation. Take $y_i^u$, add independent noise $\epsilon$ from $\mathcal{N}(0, \frac{8}{\pi})$, and binarize $y_i^u$ simply by a hard threshold 0 to get $x_i^u$. This gives the same distribution for $x_i^u$, since the probabilities match:
\begin{eqnarray*}
P(x_i^u=1) = P( y_i^u + \epsilon > 0 ) = P(\epsilon > -y_i^u)\\
= \int_{-y_i^u}^{\infty} \mathcal{N}\left( \epsilon \big| 0, \frac{8}{\pi} \right) d \epsilon
=\Phi\left(\sqrt{\frac{\pi}{8}} y_i^u \big\vert 0,1 \right).
\end{eqnarray*}

A binary ICA model  $\mathcal{M}=(\A,\{\mub_\z^u\}_u,\{\Sigmab_\z^u\}_u)$ thus consists of the following parameters: the mixing matrix $\A$, the means $\mub_\z^u$ and the diagonal (co)variance matrices $\Sigmab_\z^u$ for all segments $u$, denoted by $\{\mub_\z^u\}_u$ and  $\{\Sigmab_\z^u\}_u$. Consequently, it defines a distribution for a binary vector $\x^u$ in each segment indexed by $u$.

\section{THE LIKELIHOOD}

A surprising observation regarding the the latent variable model defined in Section~2 is that we can calculate the likelihood in closed-form by employing the multivariate Gaussian CDF. For example, the model
 defines the probability of the data vector of all ones, denoted by $\mathbf{1}$, as:
\begin{eqnarray}
&& P(\x^u=\mathbf{1}|\mathcal{M})
=\int P(\x^u=\mathbf{1}|\y^u) P(\y^u|\mathcal{M})d \y \nonumber\\
   && =  \int\Phi \left(\sqrt{\frac{\pi}{8}} \y^u \vert \mathbf{0}, \I \right) \mathcal{N}(\y^u\vert \A \mub_\z^{u}, \A \Sigmab_\z^{u} \A^{\intercal})d \y \nonumber 
\end{eqnarray}
where the univariate Gaussian CDFs are written as a multivariate Gaussian CDF $\Phi$ with an identity covariance matrix. 
The benefit of using a Gaussian CDF-based linking function comes into play here, as the value of the integral is directly a value of a multivariate Gaussian CDF \citep{waissi,sigmoidtrick}: The above formula actually specifies the probability of first drawing $\y^u$, multiplying it by $\sqrt{\pi/8}$, and then, independently, drawing a standard Gaussian variable $\n \sim \mathcal{N}(\0,\I)$ that is element-wise smaller. We therefore have:
\begin{eqnarray*}
P(\x^u=\mathbf{1}|\mathcal{M}) 
=P\left( \n - \sqrt{\frac{\pi}{8}} \y^u  < \mathbf{0} \right) 
\end{eqnarray*}
This motivates us to define a random vector $\q^u$, an important construct in the following developments, as:
\begin{eqnarray}
\q^u&=& \n -\sqrt{\frac{\pi}{8}}\y^u, \label{eq:q}
\end{eqnarray}
which is simply a noisy, re-scaled and sign-flipped version of the linear mixture $\mathbf{y}^u$. 
In fact, since $\q^u$ is the sum of two independent Gaussian random vectors, it also has a Gaussian distribution $\q^u  \sim  \mathcal{N}\left(\mub_{ \q}^u , \Sigmab_{ \q}^u \right)$ with:
\begin{eqnarray}
\mub_{\q}^u&=& -\sqrt{\frac{\pi}{8}} \A \mub_\z^u, \label{eq:muq}\\
\Sigmab_{ \q}^u &=& \I + \frac{\pi}{8} \A \Sigmab_\z^u \A^{\intercal}. \label{eq:sigma} \label{eq:sigmaq}
\end{eqnarray}
The probability of the data vector of ones in segment $u$ is, then:
\begin{eqnarray}
P(\x^u=\mathbf{1}|\mathcal{M}) = P\left( \q^u < \mathbf{0} \right) =
\Phi\left(\mathbf{0} \vert\mub_{ \q}^u, \Sigmab_{ \q}^u \right), \label{eq:prob}
\end{eqnarray}
where the cumulative distribution function of the multivariate Gaussian $\Phi$ has all variables integrated from $-\infty$ to $0$; it is readily implemented in basic packages~\citep{mvnorm}.

Similar derivation gives the probabilities for other assignments to $\x^u$. These probabilities can be expressed compactly for all value assignments as:
\begin{equation}
    P(\x^u \vert \mathcal{M}) = \Phi \left(l(\x^u ), u(\x^u ) \vert \mub_{ \q}^u, \Sigmab_{ \q}^u \right) \label{eq:probxu}
\end{equation}
in which the multivariate Gaussian probability density function is integrated from the lower bound $l(\x^{u}) $ to the upper bound $u(\x^{u})$, with the $i$th elements in the bounds defined by:
\begin{eqnarray*}
l(\x^u)[i] = \begin{cases}-\infty  \text{ if } x^u_i=1 \\
\quad\,0 \text{ otherwise}
\end{cases}
u(\x^u)[i]=  \begin{cases}\;\,0  \text{ if } x^u_i=1 \\
\infty \text{ otherwise}
\end{cases}
\end{eqnarray*}

\begin{figure}
    \centering
    \includegraphics[scale=0.25,trim={16cm 9cm 16cm 12cm},clip]{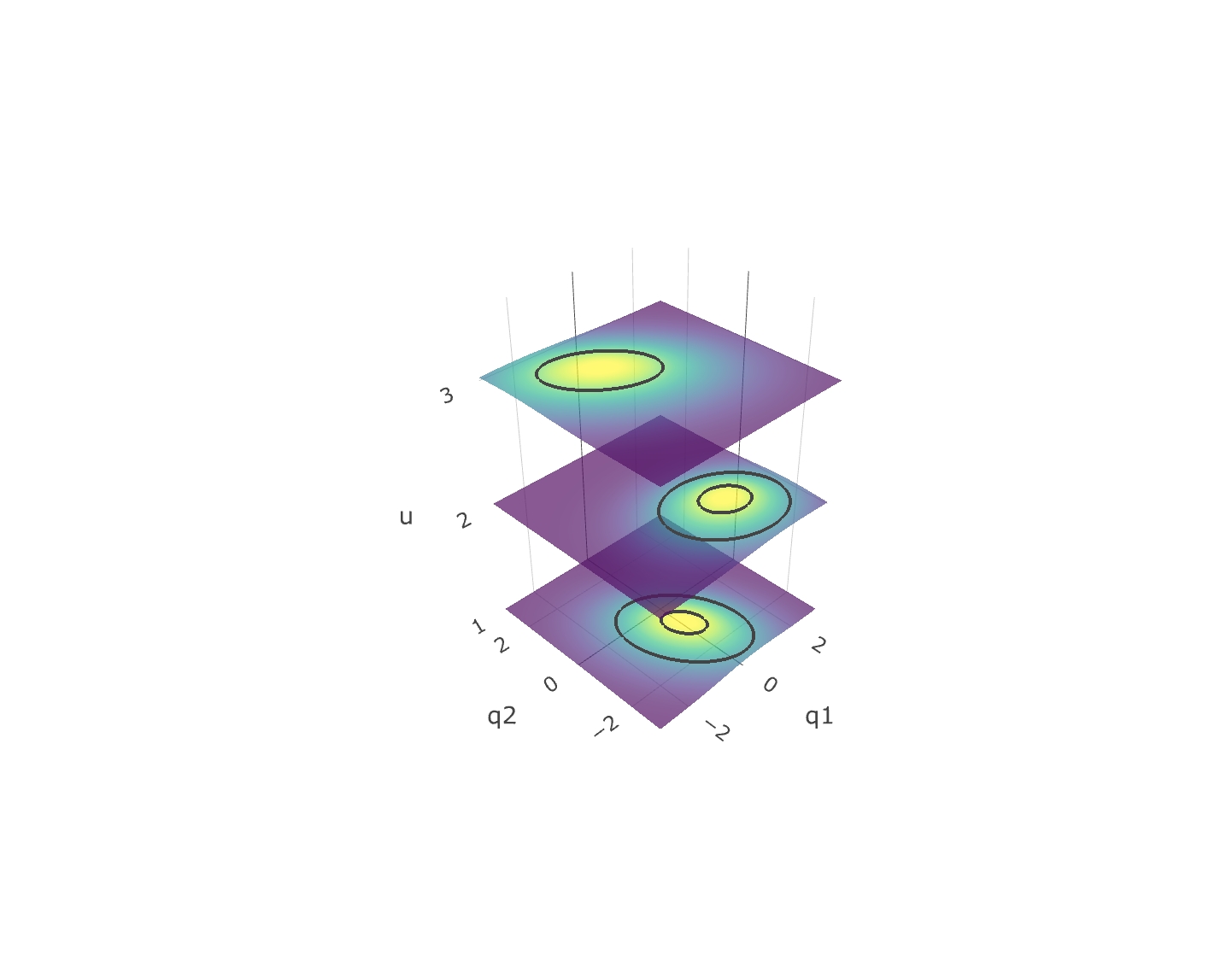}
\caption{Binary ICA model for two observed variables and three segments. For each segment, there is a bivariate Gaussian distribution on $\q^u$, the probability of an assignment to the binary observed variables is the probability mass in the corresponding quadrant.\label{fig:intuition}}
\end{figure}

Importantly, this formulation allows for a particularly clear intuitive interpretation of the model. Figure~\ref{fig:intuition} shows this for two observed variables and three segments. For each segment, the model defines a bivariate Gaussian distribution for $\q^u$, depicted by colors and contours on the planes. The probability for an assignment of the observed binary variables $\x^u$ in a segment is simply the probability mass in a corresponding quadrant. The multivariate Gaussian distributions for $\q^u$ in each segment are related in the sense that they are formed by the same mixing matrix performing on independent sources particular to the segment.

The log-likelihood of the whole data set can then be calculated as

\begin{eqnarray}
l
&=&
\sum_u
\sum_{\x^u} \hspace{-1mm} c(\x^u)\log \Phi(l(\x^u), u(\x^u) \vert \mub_{ \q}^u, \Sigmab_{\q}^u), \label{qdistr} \label{eq:likelihood}
 \end{eqnarray}
 where $c(\x^u)$ is the count of the data points with assignment $\x^u$ in a segment $u$
 and the sum is taken over all assignments to $\x^u$ and $u$.

\section{ON IDENTIFIABILITY}

Many ICA models can only be identified up to scaling and permutation indeterminacies of the sources~\citep{ICAbook,Khemakhem2019}. Straightforwardly we can see that those limitations apply for our model as well. By re-ordering columns of the mixing matrix and the sources, the implied distribution is unaffected; similarly, we can counteract the scaling (or sign-flip) of the mixing matrix columns by scaling (or sign-flipping) the sources. However, binarization actually induces additional indeterminacies as we will show next. 

\subsection{The Binarization Indeterminacy}

Recall that the probability of an assignment to binary $\x^u$ is given by the probability of the Gaussian $\q^u$ landing in different regions (Equation~\ref{eq:prob}).
But note that the probability in Equation~\ref{eq:prob} stays exactly the same even if $\q^u$ is multiplied by a diagonal matrix $\Qb^u$, possibly different for each segment $u$, with positive entries (scaling factors) on the diagonal:
\begin{eqnarray*}
P\left( \q^u  < \mathbf{0} \right) &=&P\left( \Qb^u  \q^u < \mathbf{0} \right).
\end{eqnarray*}
This is valid even if the elementwise operator is $>$ or a mixture of $>$ and $<$.\footnote{For the probability of $\x^u$ being all ones, any permutation matrix $\Qb^u$ would similarly preserve the implied probability, but the probability of some other assignment for $\x$ (each of which corresponds to some mixture of $>$ and $<$) may change then.}
Figure~\ref{fig:indet} shows an example of this equivalence relation for one segment and two observed variables. The two Gaussian distributions for $\q^u$ represented by the blue and red contours imply the exact same joint distribution for binary observed variables $\x^u$. The amount of mass in each of the 4 quadrants is exactly the same.  This means that we essentially lose all scale information on $\q^u$ in the binarization.

Then, two binary ICA models $\mathcal{M}=(\A,\{\mub_\z^u\}_u,\{\Sigmab_\z^u\}_u)$ and $\hat{\mathcal{M}}=(\hat{\A},\{\hat{\mub}_\z^u\}_u,\{\hat{\Sigmab}_\z^u\}_u)$ are indistinguishable if there are positive diagonal matrices $\{\Qb^u\}_u$ such that for each segment $u$, the means and covariances of $\q^u$ satisfy: 
\begin{eqnarray}
 \hat{\mub}_{\q}^u &=& \Qb^u \mub_{\q}^u \label{eq_1}, \\
\hat{\Sigmab}_{ \q}^u  &=& \Qb^u\Sigmab_{\q}^u \Qb^u,
\label{eq_2}
\end{eqnarray}
which can be written more clearly using the model parameters (Equations~\ref{eq:muq}  and~\ref{eq:sigmaq}) as:
\begin{eqnarray}
    \sqrt{\dfrac{\pi}{8}}\hat{\A} \hat{\mub}_\z^u&=& \Qb^u \sqrt{\dfrac{\pi}{8}} \A \mub_\z^u,
\label{eq_1v} \label{arithmetic1}\\
    \I + \dfrac{\pi}{8} \hat{\A} \hat{\Sigmab}_\z^u (\hat{\A})^\intercal &=& \Qb^u(\I + \dfrac{\pi}{8} \A \Sigmab_\z^u \A^\intercal )\Qb^u.
\label{eq_2v} \label{arithmetic2}
\end{eqnarray}
This limits identifiability possibilities (Section~4.2) but nevertheless also allows for the development of efficient estimation procedures in Sections~4.3 and~5.2.

\begin{figure}
    \centering
    \includegraphics[scale=0.50]{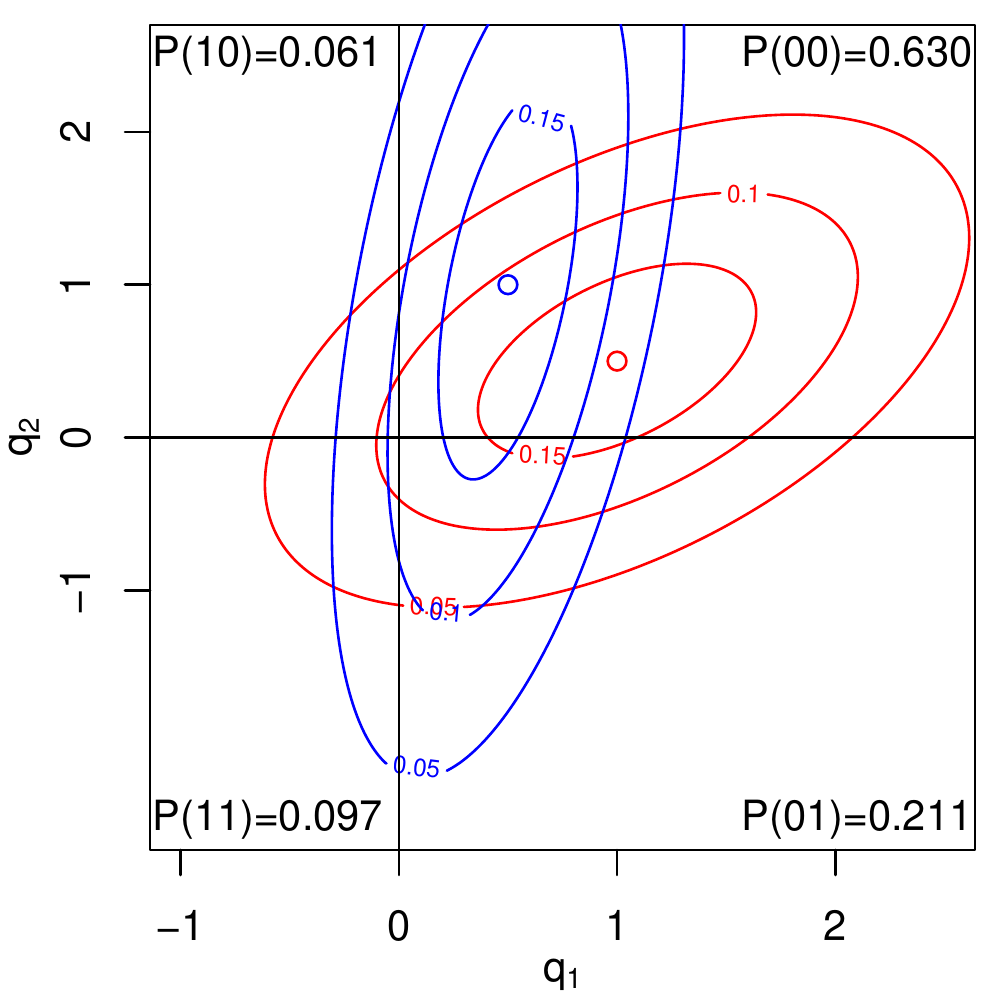}
\caption{Two Gaussian distributions (red and blue) for a two dimensional $\q^u$ which imply the same binary distributions after binarization by the linking function. That is because the mass of both distributions in each of the 4 quadrants is identical.\label{fig:indet} }
\end{figure}

\subsection{The Row Order Indeterminacy}

One of the consequences of the binarization indeterminacy is the following non-identifiability result concerning $n=2$ observed variables, proven in Appendix A in the supplement.
\begin{theorem}
If the row order of the 2-by-2 mixing matrix $\A$ of a binary ICA model is reversed, then the source means $\mub^u_\z$ and variances $\Sigmab^u_\z$ can be adjusted such that the implied distributions for the observed binary $\x^u$ remain identical.
\end{theorem}

This means that in addition to column order and scale, we also have row order indeterminacy here. Although the result may generalize to certain sparse higher dimensional models, fortunately, it does not jeopardize the estimation of higher dimensional models in general. 

This result does have consequences for 
causal discovery~\citep{Shimizu06JMLR,suzuki2021causal,peters2011causal,inazumi2014causal}.
Consider two structural equation models, implying opposite causal directions:
$$
\y^u= \left( \begin{array}{cc}
0 & 0\\
b & 0
\end{array}\right)\y^u + \z^u, \quad \y^u:= \left( \begin{array}{cc}
0 & b\\
0 & 0
\end{array}\right)\y^u + \z^u.
$$
where $\z^u$ has a Gaussian distribution in each segment $u$ with diagonal covariance matrix $\Sigmab_\z^u$.  The models 
correspond respectively to the mixing models (compare to Equation~\ref{eq:mixing}):
$$
\y^u=\left( \begin{array}{cc}
1 & 0\\
b & 1
\end{array}\right)\z^u, \quad \y^u=\left( \begin{array}{cc}
1 & b\\
0 & 1
\end{array}\right)\z^u.
$$
If we observed binarized $\y^u$, i.e. $\x^u$, we can at most identify the mixing matrix up to row order, column order and column scale.
By switching the column order and then the row order of the mixing matrix on the left, we get the mixing matrix on the right. Thus, unlike in the continuous case, we cannot detect the causal direction between two variables without further limiting assumptions or information on other variables.

\subsection{The Correlation Identifiability}

Note that the indistinguishable models satisfying Equation~\ref{eq_2} or Equation~\ref{eq_2v} have equal \emph{correlation matrices} (i.e.\ matrices of Pearson correlation coefficients) for the random variables $\q^u$. 
The next theorem  and corollary show that the correlations between elements of $\q^u$ are indeed theoretically identifiable from the distributions of the binary observed variables $\x^u$. Intuitively, the higher the correlation, the more likely will the pair of binary observed variables in $\x^u$ receive equal assignments. The fairly technical proof is given in Appendix B in the supplement.

\begin{theorem}
Two binary ICA models imply different distributions for binary observations $\x^u$ (in a given segment $u$) if the correlation matrices for $\q^u$ are not equal.
\end{theorem}

This result is crucial for the development of our novel estimation method (Section~5.2), via the corollary:

\begin{corollary} \label{corollary1}
The correlation matrix of $\q^u$ in a given segment $u$ is identifiable from the distribution for binary $\x^u$.
\end{corollary}

On the other hand, the following theorem recaps the well-known result \citep{ICAbook,Pham01} that the means do not help in estimating the mixing matrix  (proven in Appendix~C):
\begin{theorem}
If two models
$\mathcal{M}$
and $\hat{\mathcal{M}}$ with $n=n_z$
imply the same correlation matrices for $\q^u$ (in a given segment)
 then the means $\mub_\z^u$ can be adjusted such that the implied binary distributions are identical. \label{thm:means}
\end{theorem}

\section{METHODS FOR BINARY ICA}

Next, we present three methods for estimating the binary ICA model, building on the theory in Sections~3 and~4. The \texttt{BLICA} method of Section~5.2 is the main novel algorithmic contribution of the paper.

\subsection{Maximum Likelihood Estimation}

We have already derived the likelihood of the binary ICA model in Equation~\ref{eq:likelihood}. A straightforward approach is then to optimize this using e.g. L-BFGS~\citep{lbfgs}. The gradient involves the moments for the \emph{truncated} multivariate Gaussian distribution,
which can be obtained from R package \texttt{tmvtnorm}~\citep{tmvtnorm}.
Variances and scaling factors can be kept positive by using the log-exp transform. Unfortunately, the computation of the likelihood and its gradient can only be done for small models in practice, because the evaluation of the multivariate Gaussian CDF is time consuming, necessitating the use of sampling-based approximations. Our experiments refer to this as \texttt{full MLE}.

\subsection{The \texttt{BLICA} Method}

However, we can circumvent the computational burden of the high-dimensional Gaussian cumulative distribution function.
Due to the theory in Section 4, the correlations of $\q^u$ convey the essential information between the binary data and the continuous mixing model. Since the marginalization properties of our model are inherited from the multivariate Gaussian, such correlations can be estimated from \emph{pairwise} marginal distributions of elements of $\x^u$; in 2D the Gaussian cumulative distribution function is still quite quick to compute.
Thus, we combine maximum likelihood estimation with what could be called a ``moment-matching'' approach as follows.
We first recover the pairwise correlations of the continuous-valued $\q^u$ from the observed binary data on $\x^u$ (this is possible by Corollary~\ref{corollary1}) via MLE in 2D. Then we fit those correlations to the correlations implied by the latent linear mixing model using a more scalable MLE in the continuous-valued latent space. The resulting algorithm is summarized as Algorithm~\ref{alg:BLICA} and explained in detail below.

\textbf{Correlation estimation.}
On line 4, we estimate each correlations between elements in $\q^u$ separately, by directly fitting the likelihood in Equation~\ref{qdistr} in two dimensions, thus estimating $\mub_\q^u$ and $\Sigmab_\q^u$.
To calculate the multivariate Gaussian CDF, we use the R package \texttt{mvtnorm}~\citep{mvnorm}.
We employ the \texttt{GenzBretz} method, which is particularly suitable for the fast evaluation needed here~\citep{genz1993comparison}. Furthermore, the estimation can be simplified~\citep{lee}. Due to Equation~\ref{arithmetic2} the diagonal of $\Sigmab^u_\q$ can be set to 1s in this step. 
Furthermore, since the marginal of $x_i^u$ is
\begin{eqnarray}
P(x_i^u=1)&=&\Phi(-\mub^u_\q[i]/\sqrt{\Sigmab^u_\q[i,i]}|0,1), \label{eq:meanest}
\end{eqnarray}
both means in $\mub^u_\q$ can be computed from the respective marginals using the 1D inverse Gaussian CDF separately~\citep{mvnorm}. The univarite optimization problem for the remaining parameter in the interval $[-1,1]$ can then be solved efficiently using a line search method~\citep{brent2013algorithms}. The scalability of Algorithm~1 depends crucially on this step, as $n_u\cdot (n^2-n)/2$ correlations need to be estimated. The separately estimated correlations are collected to $n_u$ segmentwise $n$-by-$n$ correlation matrices denoted by $\C^u_\q$. 

\begin{algorithm}[!t]
\begin{algorithmic}[1]
\State Input data recorded at $n_u$ different segments.
\For{ segment $u \in \{1,\ldots,n_u\}$ }
\For{ each observed variable pair $\{x_i^u,x_j^u\}$ } 
\State \parbox[t]{6cm}{Estimate the correlation between $q_i^u$ and $q_j^u$ by maximizing the marginal pairwise likelihood of  $x_i^u$ and $x_j^u$ (in segment $u$).}
      \EndFor
      \State \parbox[t]{7cm}{Form and regularize the correlation matrix $\C_\q^u$ obtained from the pairwise correlations.}
\EndFor
\State Optimize scaled Gaussian likelihood  
with L-BFGS
over sufficient statistics  $\C_\q^u$ from all segments $u$.
\State Return the estimated mixing matrix $\A$ and source variances $\Sigmab^u_\z$ for all segments $u$.
\end{algorithmic}
\caption{The \texttt{BLICA} algorithm for Binary ICA.\label{alg:BLICA}}
\end{algorithm}

\textbf{Regularization.} When estimating the correlations of $\q^u$ from sample data, it can happen that a correlation matrix $\C_\q^u$ is close to singular or not positive definite. We use the following regularization on line 5, based on the parameter $r$~\citep{warton}, which marks the approximate condition number targeted. 
The regularized correlation matrix is then
\begin{equation}
\frac{1}{1+\delta} (\C_\q^u + \delta \mathbf{I}), \text{ where } \delta = \max \left( 0, \frac{\lambda_1-r\cdot \lambda_n }{ r-1} \right), \nonumber
\end{equation}
where $\lambda_1$ is the largest and $\lambda_n$ the smallest eigenvalue of $\C_\q^u$.
This regularization keeps the unit diagonal.

\textbf{Moment Matching.} 
Finally, on line 6, we fit the model parameters (including stationary $\A$) to the estimated correlations $\C^u_\q$ using a Gaussian likelihood model 
over the different segments $u$ (Section~2). But in contrast to the usual case where we have the covariance matrices, here we need to account for the ``binarization indeterminacy'', resulting in additional nuisance scaling parameters, as pointed out above. We use the term scaled Gaussian likelihood to refer to the ordinary multivariate Gaussian likelihood where we include additional parameters $\Qb^u$ as the scaling factors.
The fitting is thus done by the following scaled Gaussian likelihood based on the sufficient statistics $\C_\q^u$:
\begin{equation}
l =\sum_{u=1}^{n_u} \frac{N}{2} \left[
   -\log (\det (\Sigmab_\q^u)) - \Tr(\C_\q^u (\Sigmab_\q^u)^{-1}  ) \right] \nonumber
\end{equation} 
where recall that $\Sigmab_\q^u=\Qb^u(\mathbf{I}+\A\Sigmab_\z^u \A^T)\Qb^u$ by Equation~\ref{arithmetic2} is a function of
 the mixing matrix $\A$, 
 source variances $\{\Sigmab_\z^u\}_u$ (diagonal, positive elements) and scaling factors $\{\Qb^u\}_u$ (diagonal, positive elements). Variances and scaling factors can be kept positive by using the log-exp transform. Note that without the scaling factors $\{\Qb^u\}_u$, the mixing matrix $\A$ could be found via joint diagonalization~\citep{JSSv076i02}. Note also that due to Theorem~\ref{thm:means}, the source means do not need to be estimated.
 Here, instead, we perform the fitting by maximizing this likelihood using L-BFGS~\citep{lbfgs} with respect to the aforementioned parameters.

\subsection{Binary ICA through Linear iVAE} \label{ivae}

\citet{Khemakhem2019} presented the identifiable Variational Autoencoder (iVAE), an approach for nonlinear ICA employing variational autoencoders \citep{Kingma2014, Rezende2014} that assumes access to an additionally observed variable such that the sources are independent given the auxiliary variable; further, each source follows an exponential family distribution given the auxiliary variable. 
Here, we apply the iVAE approach to estimate the binary ICA model from Section~\ref{model} \citep{barin2021independent}. As proposed by \citet{Kingma2014} and \citet{ Khemakhem2019}, we use the factorized Bernoulli observational model and apply a sigmoid function element-wise to the output of the decoder to obtain the binary probabilities. 
Due to the linearity of our mixing model and the segment-wise structure, we can simplify the encoder (posterior approximation) of the VAE, and make all the transformations in the iVAE affine or linear, thus greatly simplifying the system. The \texttt{linear iVAE} is presented in more detail in Appendix~E.

\subsection{Estimation of the Sources}

After estimating the mixing matrix $\A$, it may be desired to estimate the sources $\z^u$ as well.
In the case of binary data, the individual source values cannot be accurately estimated (even up to scale and order indeterminacies) due to the inherent noise introduced by the binarization procedure. Presumably, though, if the number of observed variables is large and the number of sources is small, the estimation may be reasonable. In any case, the posterior $P(\mathbf{z}^u|\mathbf{x}
^u)$ can be easily calculated after estimating the mixing matrix.

\section{EXPERIMENTS}

\begin{figure*}[!t]
    \centering
\includegraphics[scale=0.42]{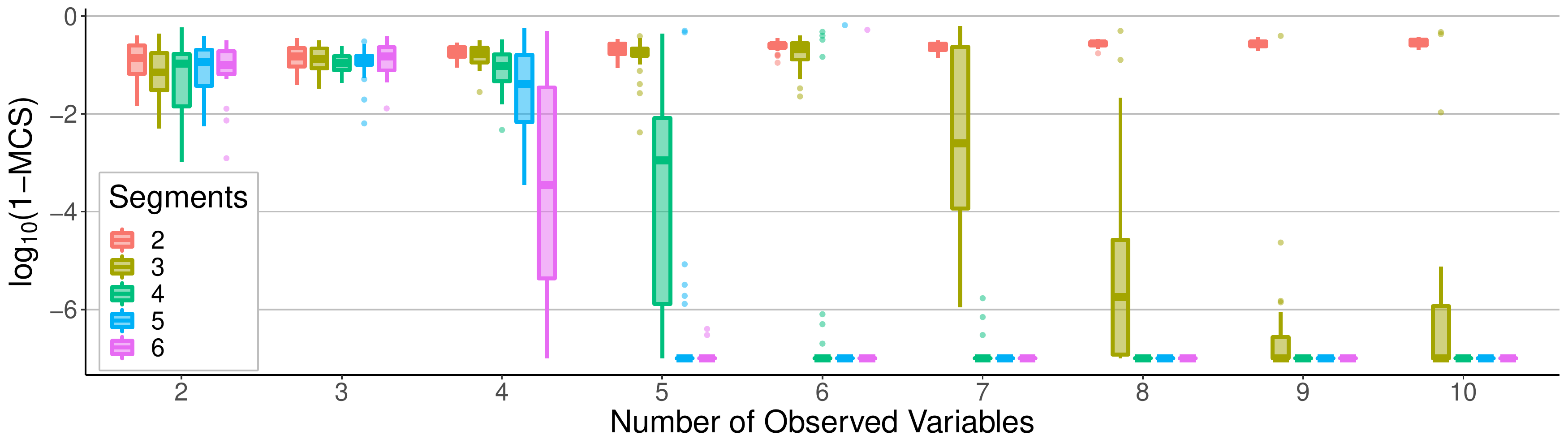}
\caption{Identifiability 
with equal number of observed variables and sources. The \texttt{BLICA} method used true (pairwise) probability distributions (i.e. infinite sample limit data). Each box is based on 30 models. 
A lower value on the y-axis (log-error) implies better performance.
Runs with values less than $-7$ (e.g. those in which the model was identified up to machine precision) are marked with $-7$. Compare to Table~1.\label{fig:idplot} }
\end{figure*}

We implemented our proposed methods and baselines using R (\texttt{BLICA}, \texttt{full MLE}) and python (\texttt{linear iVAE}). 
Here we investigate 
the identifiability of the model, as well as the finite-sample estimation performance and the scalability of our 
methods, also comparing to previous approaches. 

\textbf{Data.} The data was generated from the Binary ICA model (Section~2)
 in the following way. Means were drawn from $\mathrm{unif}(-0.5,0.5)$, standard deviations from
$\mathrm{unif}(0.5,3)$. Mixing matrix elements were drawn from $\mathrm{unif}(-3,3)$ while ensuring invertibility by resampling until the condition number ($\kappa$) was below 20 for $n<20$, or for $n\leq 20$ below the 75th quantile of 1000 sampled similar dimensional mixing matrices. For practical estimations from finite sample data we use 40 segments, varying the sample size per segment.

\textbf{Evaluation.} 
ICA methods are often compared in terms of the mean correlation coefficient of the estimated sources. Here, however, binarization induces heavy noise and individual samples of the estimated sources cannot be accurately estimated. We therefore focus our evaluation on the mixing model, 
and measure the mean cosine similarity (MCS) of the mixing matrix columns (taking the inherent order and scale indeterminacy of the sources into account, see Appendix~D).

\subsection{Identifiability}

\textbf{Results.}
Recall from Sections~4 and~5 that the correlations of $\q^u$ convey the information between the binary data and the mixing model, and each of these correlations can be determined from the marginal distributions over the corresponding pair of binary observed variables in $\x^u$(in a segment $u$).
Thus,  by using the exact pairwise binary distributions of elements of $\x^u$ from Equation~\ref{eq:probxu} as input for BLICA, 
we are here able to investigate  identifiability empirically without any finite sample effects.
Figure~\ref{fig:idplot} shows results on which models can be identified when the number of sources and observed variables are equal ($n\hspace*{-1mm}=\hspace*{-1mm}n_z$). In many cases, the method found the mixing matrix essentially up to machine precision, which can be seen as indication of identifiability.
Each box includes 30 different data generating models, and for each we ran \texttt{BLICA} 3 times; the MCS of the run with highest scaled Gaussian likelihood is plotted. 
With only 2 segments, or only 2 observed variables (also in Theorem~3), the model is not identifiable in any case. The minimal cases deemed identifiable (up to source scale and order) are $(n\hspace*{-1mm}=\hspace*{-1mm}5,n_u\hspace*{-1mm}=\hspace*{-1mm}5)$, $(n\hspace*{-1mm}=\hspace*{-1mm}6,n_u\hspace*{-1mm}=\hspace*{-1mm}4)$, $(n\hspace*{-1mm}=\hspace*{-1mm}7,n_u\hspace*{-1mm}=\hspace*{-1mm}4)$, $(n\hspace*{-1mm}=\hspace*{-1mm}8,n_u\hspace*{-1mm}=\hspace*{-1mm}4)$, $(n\hspace*{-1mm}=\hspace*{-1mm}9,n_u\hspace*{-1mm}=\hspace*{-1mm}3)$, and $(n\hspace*{-1mm}=\hspace*{-1mm}10,n_u\hspace*{-1mm}=\hspace*{-1mm}3)$.
Thus generally, the more observed variables ($n$) we have, the less segments ($n_u$) are needed.

\begin{table}[!t]
\centering
{
\setlength{\tabcolsep}{2.5pt}
\begin{tabular}{c|rrrrrrrrr}
Number of & \multicolumn{9}{c}{Number of Observed Variables ($n$)}\\
Segments ($n_u$) & $2$ & \phantom{0}$3$ & $4$ & $5$ & $6$ & $7$ & $8$ & $9$ & $10$\\ 
  \hline
$2$ & -6 & -9 & -12 & -15 & -18 & -21 & -24 & -27 & -30 \\ 
  $3$ & -7 & -9 & -10 & -10 & -9 & -7 & -4 & {\color{red}\textbf{0}} & {\color{red}\textbf{5}} \\ 
  $4$ & -8 & -9 & -8 & -5 &  {\color{red}\textbf{0}} & {\color{red}\textbf{7}} & {\color{red}\textbf{16}} & 27 & 40 \\ 
  $5$ & -9 & -9 & -6 & {\color{red}\textbf{0}} & 9 & 21 & 36 & 54 & 75 \\ 
  $6$ & \phantom{0}-10 & -9 & -4 & {\color{red}\textbf{5}} & 18 & 35 & 56 & 81 & 110 
\end{tabular}
}
\caption{Heuristic identifiability analysis. 
Each entry states the number of statistics (equations) minus the number of unknowns. 
The minimal cases 
with a non-negative number,
suggesting identifiability, are bolded in red.
\label{tab:idtable}}
\end{table}

\begin{figure*}
    \includegraphics[scale=0.36]{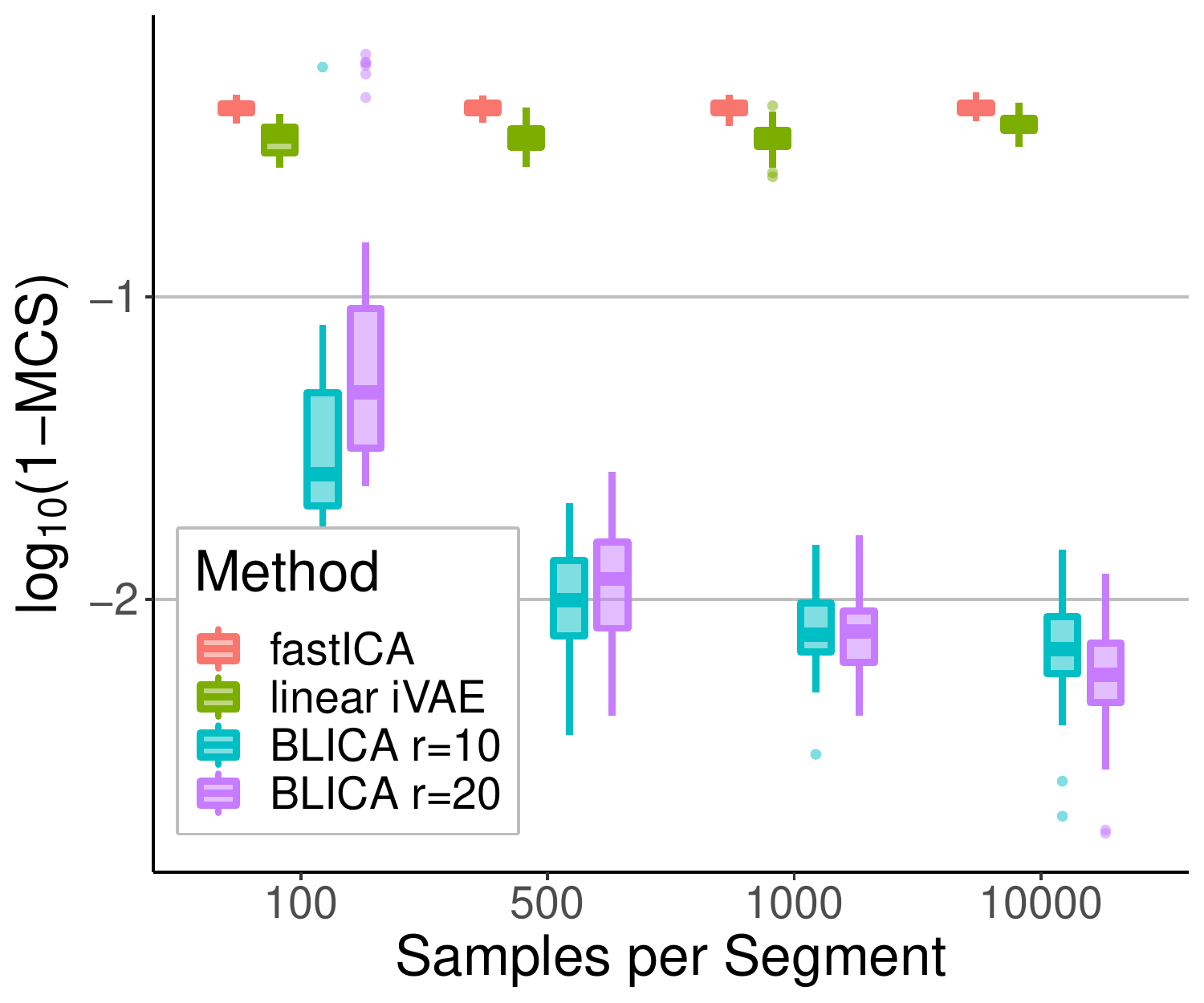}\;\includegraphics[scale=0.36]{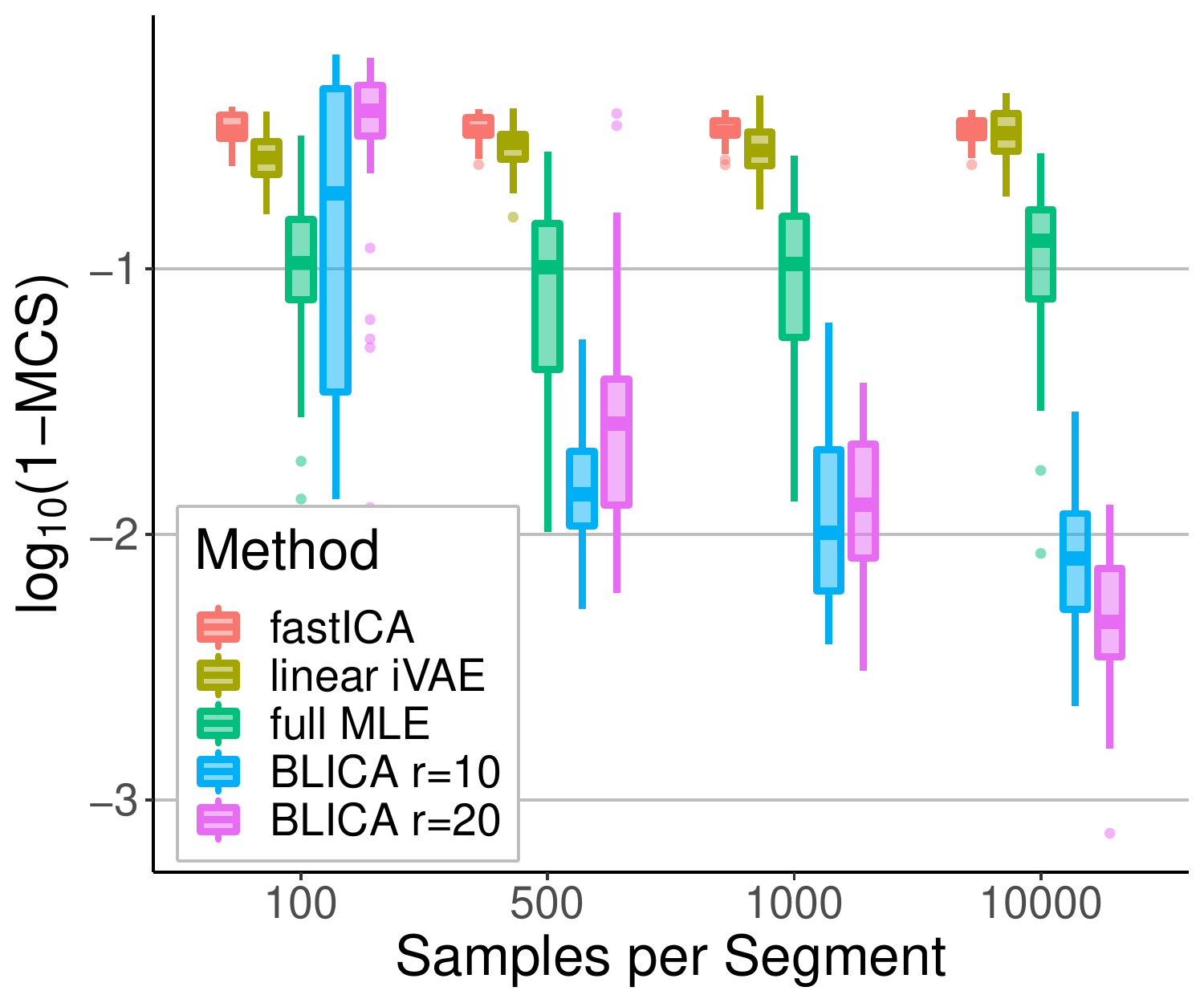}\;\includegraphics[scale=0.36]{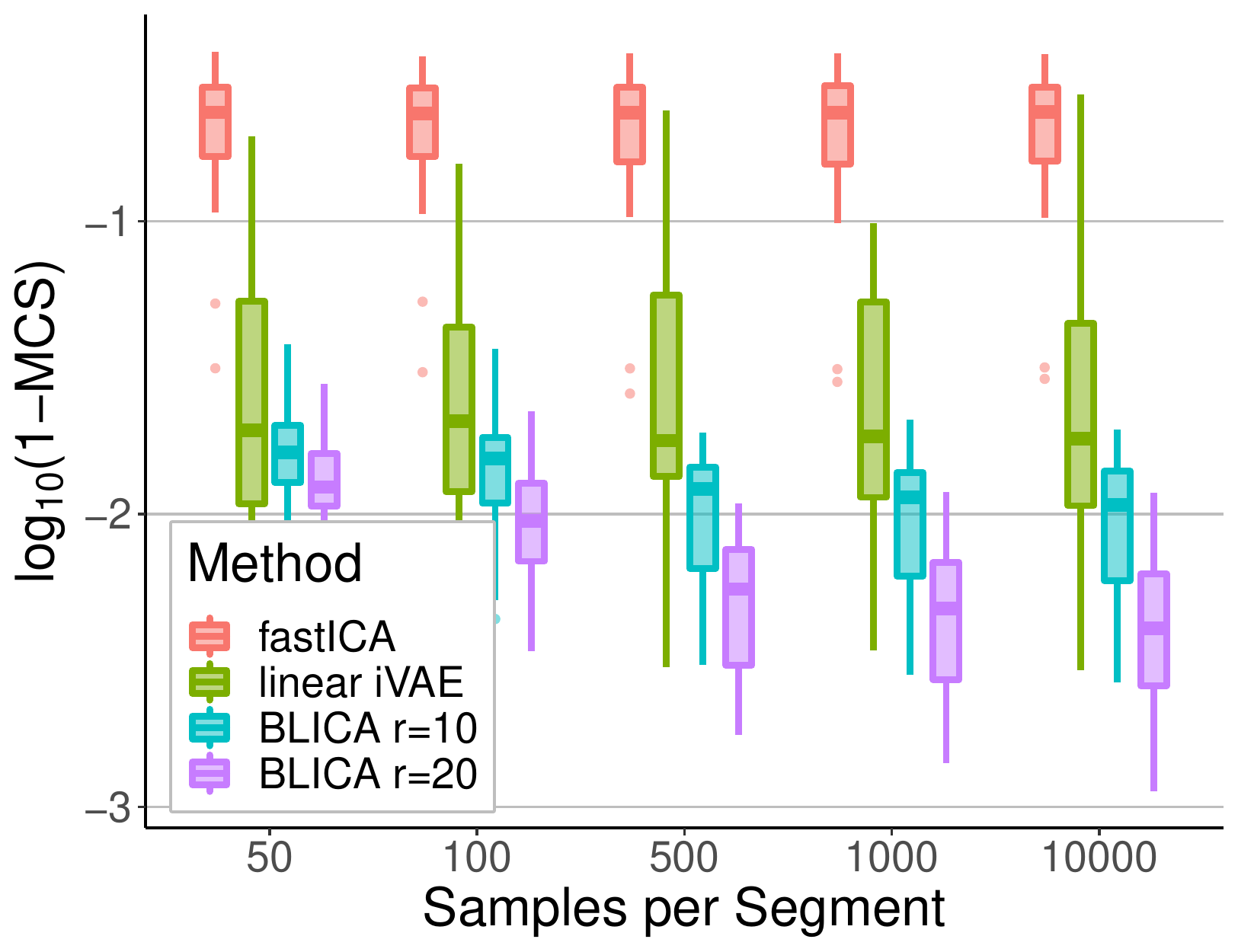}
\caption{
Finite sample performance. 
Left: 10 observed variables and 10 sources.
Center: 6 observed variables and 6 sources.
Right: 6 observed variables and 2 sources. 
Each box is based on thirty 40-segment datasets. 
\label{fig:10observed}\label{fig:comparison} }
\end{figure*}

\paragraph{Heuristic Identifiability Analysis.} We contrast 
the results to the 
well-known heuristic approach to identifiability used in factor analysis. It is based on counting the number of statistics we can calculate (or equations we can form), and the number of unknowns (parameters) we need to solve. If the 
former
is at least as large as the 
latter, there is hope that the model is identifiable.
The calculations in Table~\ref{tab:idtable} are based on Equations~\ref{arithmetic1} and~\ref{arithmetic2} when the number of sources equals the number of observations ($n=n_z$).
The statistics correspond to $n_u(n^2-n)/2$  covariances, $n_u\cdot n$ variances and $n_u\cdot n$ means (for $\q^u$). 
Unknowns include $n\cdot n$ mixing matrix coefficients, $n_u\cdot n$ (segment-wise) source variances, $n_u\cdot n$ source means, as well as $n_u\cdot n$ scaling terms (diagonal elements of $\Qb^u$).
In line with the classical literature in factor analysis, we ignore the source 
order indeterminacy. Figure~\ref{fig:idplot} and Table~\ref{tab:idtable} show a remarkably similar dependence between identifiability and the numbers of the segments and the observed variables: in particular, they agree on the minimal cases identifiable. Interestingly, cases with 2 observed variables as well as the cases with only 2 segments are never identifiable.
Note that these computational results together with Section~4 provide a bound for any future analytical results on identifiability.
If identifiability turns out to be possible in further cases, e.g., with a different mixing model or linking function, the results will need to depend on the particular parametric forms, thus limiting applicability.

\subsection{Finite Sample Estimation}

\textbf{Methods.} Next we turn our attention to estimation performance from finite sample data.
 We compare our new  \texttt{BLICA} (with different regularization parameter value $r$) method to its main competitors, \texttt{fastICA}~\citep{Himberg01,fastica} and the baseline implementations of \texttt{linear iVAE} and \texttt{full MLE}. Note that the model of \texttt{fastICA} is somewhat different, but it still employs a linear mixing of the sources and has the same sources scale and order indeterminacies; thus, \texttt{MCS} comparison is sensible. \texttt{fastICA} does not use the segment index, but pools all data from different segments.
Recall from Section~\ref{ivae} and Appendix~E that the \texttt{linear iVAE} uses essentially the same model, but instead of employing the likelihood, it optimizes the ELBO objective through L-BFGS. For runs with $n<20$ observed variables, a time budget of 2h was used, and the results that were obtained within the time limit are reported. For larger simulations, we allowed for 12h per run. To avoid local minima due to the difficult optimization landscape, we ran the \texttt{linear iVAE}, \texttt{full MLE} and \texttt{BLICA} with 3 different learning seeds and selected the best run according to the objective function (e.g. likelihood). 

\begin{figure*}[!t]
    \centering
    \includegraphics[scale=0.36]{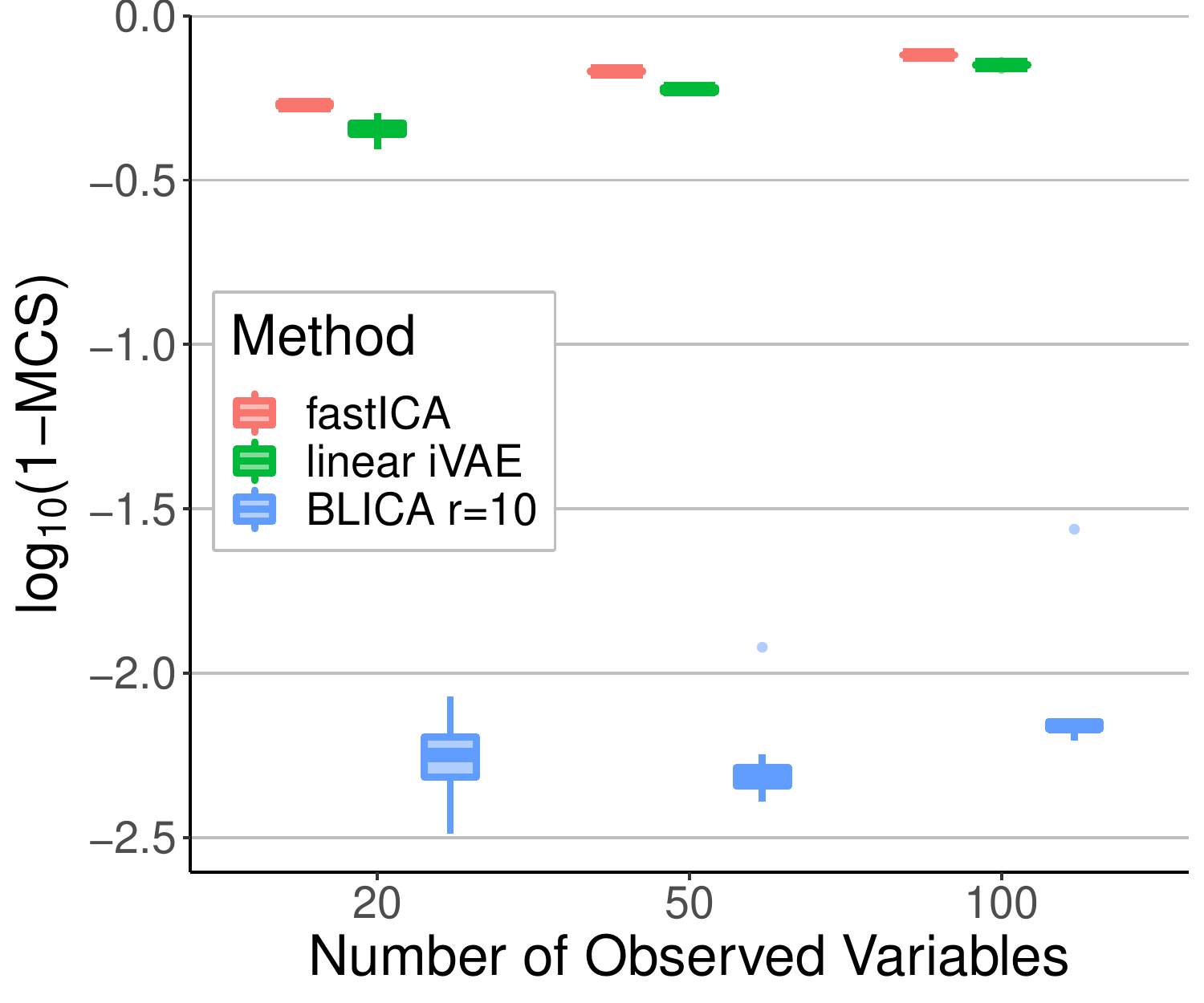}\;\includegraphics[scale=0.36]{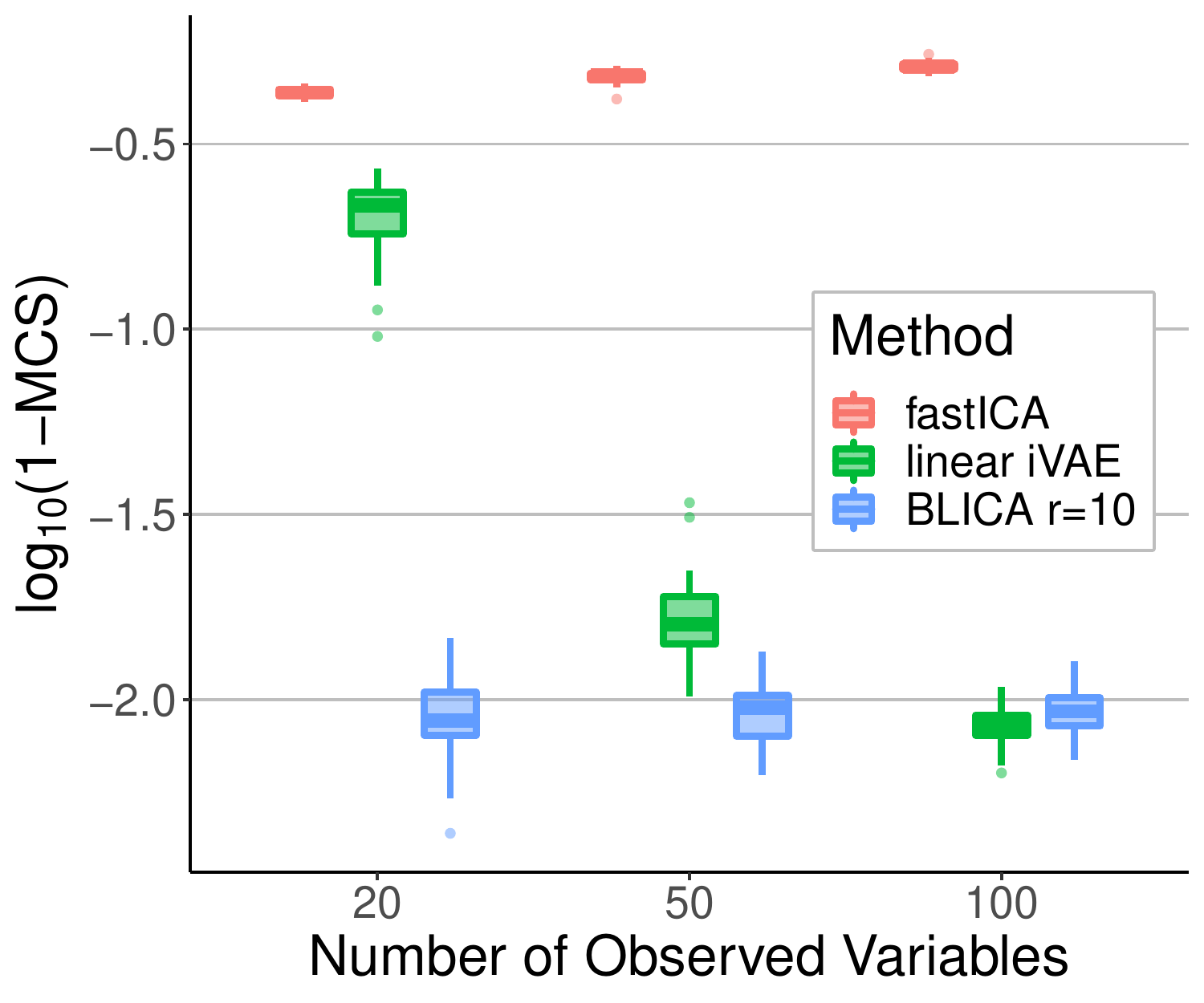}\;\includegraphics[scale=0.36]{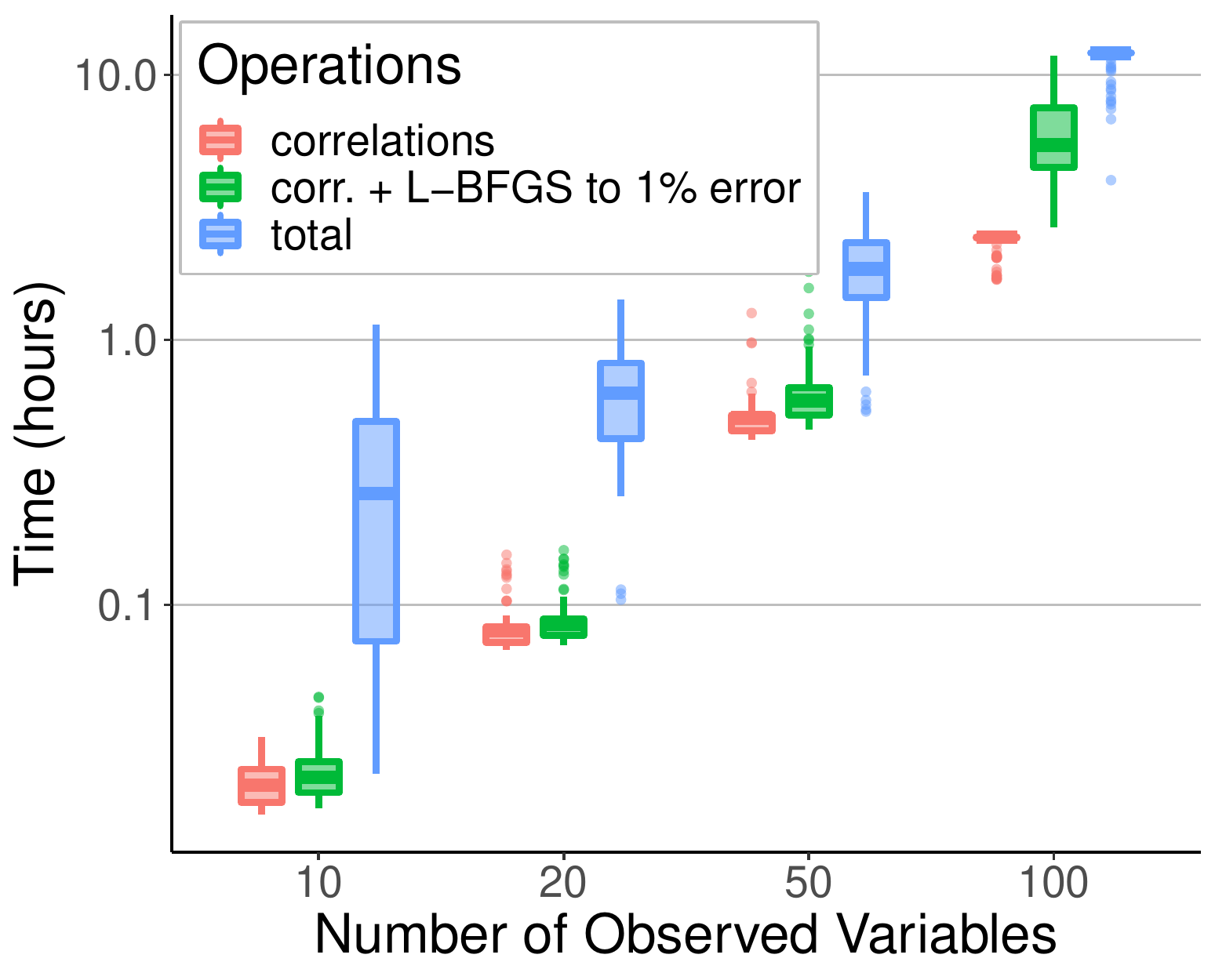}
\caption{Scalability.
Left: equal number of sources and observed variables. Center: 10 sources. 
Each box is based on thirty 40-segment datasets with 1000 samples per segment.
Right: Running times of the  steps of \texttt{BLICA} (Algorithm~1).
\label{fig:highdim} }
\end{figure*}

\textbf{Results.}
 Figure~\ref{fig:10observed} (left) shows the result for 10 observed variables and 10 sources. \texttt{BLICA} clearly outperforms others consistently improving with increasing sample size. 
With smaller dimensions, 6 observed variables and 6 sources in Figure~\ref{fig:10observed} (center), BLICA needs more samples to achive similar MCS. However, with fewer sources fewer samples are needed:  Figure~\ref{fig:10observed} (right) shows that for 6 observed variables and 2 sources, high MCS can be obtained with only 50 samples per segments. Interestingly, \texttt{linear iVAE} performs well only with fewer sources than observations, while \texttt{fastICA} is not able to reliably estimate the mixing matrix from binary data. Unfortunately, \texttt{full MLE} cannot perform sufficiently many optimization steps within the time limit of 2h even with 6 observed variables in Figure~\ref{fig:10observed} (center).

\textbf{Scalability.} Figure~\ref{fig:highdim} assesses the performance in higher dimensions over data sets with 40 1000-sample segments, thirty for each $n$. 
Only \texttt{BLICA} can estimate the mixing matrix with equal number of observed variables equals and sources in Figure~\ref{fig:highdim} (left). When the number of sources is fixed to 10 in Figure~\ref{fig:highdim} (center), also \texttt{linear iVAE} shows improving performance with increasing number of observed variables. Finally, Figure~\ref{fig:highdim} (right) shows the running time performance of \texttt{BLICA} (Algorithm~1) on the previous runs. The estimation of the quadratic number of correlations  starts taking considerable time with 100 observed variables. L-BFGS is relatively quick in solving the optimization problem to a solution close to the final result (i.e. 1\% lower MCS), then still gradually improving.

\section{RELATED WORK} \label{sec:related}

Our research connects particularly to the following earlier and more recent literature.
\citet{Himberg01}
consider binary observed vectors $\x$ and binary sources $\z$, so that the ICA mixing model is given by the Boolean expression $ x_{i}=\bigvee_{j=1}^{n_z} a_{i j} \wedge z_{j}$.
They show that this Boolean OR mixing can be approximated by a linear mixing model followed by a unit step function. Thus, they propose to estimate the model by ordinary ICA, and obtain reasonable results when the data is very sparse.
Similarly, \citet{nguyen10} studied binary ICA with OR mixtures by defining a disjunctive generative model.
They prove identifiability  
and propose an algorithm without  continuous-valued approximations.

\citet{kaban06} proposed a model where 
continuous sources follow a Beta distribution, followed by a binary observation model.
While their approach is related to ours, their latent variables are restricted to a finite interval, and they estimate the model using variational approximation which is unlikely to yield consistent estimators.
Discrete ICA has further been approached by extensions of LDA where the topic intensities are mutually independent \citep{Podosinnikova15, Buntine05, canny04}. Although their identifiability guarantees are limited \citep{podosinnikova16}, their method has the advantage of allowing for discrete data. 
\cite{lee} consider PCA 
employing a binarized Gaussian model.

Finally, we note that the very idea of estimating latent variable models by non-stationarity, originating in \citep{Matsuoka95,Pham01}, has been recently increasingly used in estimating generative models \citep{tcl,Khemakhem2019} as well as for causal discovery \citep{zhang2017causal,Monti19UAI}, even in deep learning. Automatically estimating the segment index by a HMM has been further proposed by \citet{halva2020hidden}.
Instead of the wide-spread idea of joint diagonalization of covariance matrices~\citep{belouchrani1997blind,Tsatsanis}, we used correlation matrices without explicit diagonalization criteria; related work on diagonalizing correlation matrices can be found in \citep{corrdiag}.

\section{CONCLUSION}

We presented a model for ICA of binary data which is based on a linear latent mixing model and non-stationarity of the sources. 
We investigated the identifiability, showing some surprising indeterminacies not present in ordinary ICA, including the fact that in the two-variable case the model cannot be identified. We believe that our identifiability results, theoretical and empirical, will be useful in future research on binary ICA. Based on our approach using a Gaussian link function, the likelihood can be obtained in closed form although the Gaussian cumulative distribution function is still computationally heavy. These advances allowed for a practical method \texttt{BLICA} that combines maximum likelihood estimation and moment-matching; it was shown to be  applicable in higher dimensions while still empirically showing consistent behaviour.
As future work, we aim to generalize from binary to discrete variables, consider parallelized approaches for scaling up full MLE estimation, and investigate the potential of the new learning algorithm in applications.

\subsubsection*{Acknowledgements}

The first author was supported by the Academy of Finland under grant 315771. The second author acknowledges funding from Samsung Electronics Co., Ltd. (at Mila). The third author acknowledges funding from the Academy of Finland and a CIFAR Fellowship.

\bibliographystyle{plainnat}
\bibliography{paper}

\onecolumn

\appendix

\section{Proof of the Row Order Indeterminacy (Theorem 1)}
\setcounter{theorem}{0}
\setcounter{figure}{0}
\setcounter{equation}{0}

\begin{theorem}
If the row order of the 2-by-2 mixing matrix $\A$ of a binary ICA model is reversed, then the source means $\mub^u_\z$ and variances $\Sigmab^u_\z$ can be adjusted such that the implied distributions for the observed binary $\x^u$ remain identical.
\end{theorem}
\begin{proof}
Consider two binary ICA models
$\mathcal{M}=(\A,\{\mub^u_\z\}_u,\{\Sigmab^u_\z\}_u)$ and $\hat{\mathcal{M}}=(\hat{\A},\{\hat{\mub}^u_\z\}_u,\{\hat{\Sigmab}^u_\z\}_u)$ that have $n=2$ observed variables. Let $\hat{\A}$ be $\A$ with rows switched. We define parameters $\{\hat{\mub}^u_\z\}_u$, $\{\hat{\Sigmab}^u_\z\}_u$ and scaling matrices $\{\Qb^u\}_u$ such that Equations~10 and~11 in the main paper are satisfied and therefore the binary distributions implied by both models for each segment are identical. First, let 
$\hat{\Sigmab}^u_\z=\Sigmab^u_\z$. This and the row switching of $\A$ means that the covariance matrix of $\q^u$ has just the order switched:
$\hat{\Sigmab}_{ \q}^u[2,2]= \Sigmab_{ \q}^u[1,1]$,
$\hat{\Sigmab}_{ \q}^u[1,1]= \Sigmab_{ \q}^u[2,2]$, $\hat{\Sigmab}_{ \q}^u[1,2]= \Sigmab_{ \q}^u[1,2]$ (since this matrix is symmetric).
The equations implied by Equation~9 in the main paper for each $u$ are:
\begin{eqnarray*} \Qb^u[1,1]^2 \Sigmab_{ \q}^u[1,1] &=&\Sigmab_{ \q}^u[2,2], \\
\Qb^u[2,2]^2 \Sigmab_{ \q}^u[2,2]&= &\Sigmab_{ \q}^u[1,1],\\ 
\Qb^u[1,1] \cdot \Qb^u[2,2]\cdot \Sigmab_{ \q}^u[1,2]&=& \Sigmab_{ \q}^u[1,2].
\end{eqnarray*}
These can be solved by setting 
\begin{eqnarray*}
\Qb^u[1,1]&=&\sqrt{\Sigmab_{ \q}^u[2,2]/\Sigmab_{ \q}^u[1,1]}, \\
\Qb^u[2,2] &=&\sqrt{\Sigmab_{ \q}^u[1,1]/\Sigmab_{ \q}^u[2,2]}. 
\end{eqnarray*}
Finally, solve for $\hat{\mub}_\q^u$ from Equation~10 since $\A,\hat{\A},\Qb^u$ are invertible. 

\end{proof}

\section{Proof of the Correlation Identifiability (Theorem 2)}

\begin{theorem}
Two binary ICA models imply different distributions for binary observations $\x^u$ (in a given segment $u$) if the correlation matrices for $\q^u$ are not equal.
\end{theorem}

\iftrue
We will first present the result assuming zero means for $\q^u$ since it is more approachable to the reader. Appendix Figure~1 explains this case visually. The full technical proof is given afterwards. Appendix Figures~2 and~3 explain the general case visually.

\begin{proof}[Proof assuming zero means]

We can focus here on bivariate models as the multivariate normal for $\q^u$ can be straightforwardly marginalized to the bivariate case.
Suppose the two models respectively imply:
\begin{equation}
\q^u \sim \mathcal{N}( \mathbf{0} , \Sigmab_{ \q}^u),\quad \hat{\q}^u \sim \mathcal{N}( \mathbf{0} , \hat{\Sigmab}_{ \q}^u),\label{q_eq}\end{equation}

Due to Equations~10 and~11 in the main paper we can also assume we are dealing with ``standardized'' models where the diagonals of the covariances are units for both models. 

The correlation/covariance matrices for $\q$ and $\hat{\q}$ are:
$$
\Sigmab_{\q}^u=
 \left( \begin{array}{ccc} 
 1 &\alpha \\
\alpha & 1 \end{array}\right),\quad 
\hat{\Sigmab}_{\q}^u=
 \left( \begin{array}{ccc} 
 1 &\beta \\
\beta & 1 \end{array}\right).
$$
We study the difference in the implied binary distribution by the two models by creating the Gaussian distributions for $\q^u$ and $\hat{\q}^u$ from a single standard multivariate Gaussian source. The distributions can be formed from a standard normal $\mathbf{n}\sim N(\mathbf{0}, \mathbf{I})$, for example by multiplying with matrices
$$
\A= \left( \begin{array}{ccc} 
 1 & 0 \\
\alpha
& \sqrt{1-\alpha^2}\end{array}\right), \quad \hat{\A}= \left( \begin{array}{ccc} 
 1 & 0 \\
\beta
& \sqrt{1-\beta^2}\end{array}\right)
$$
such that
$$
\q= \A \mathbf{n}, \quad \hat{\q} = \hat{\A} \mathbf{n}.
$$
 We will assume $\alpha > \beta$ without loss of generality. Let's look at which values for $\n$ result in different assignments for the binary variables. Recall that the assignment is determined deterministically by the quadrant $\q^u$ and $\hat{\q}^u$ land in. Intuitively, the model with higher correlation $\alpha$ implies more similar values for the binary variables. For the $\alpha$-model (with $\A$):
$$
x^u_1 =  \begin{cases} 0, & \text{ if }n_1 > 0 \\
1, & \text{ if }n_1 < 0
\end{cases} , \quad x_2^u = \begin{cases} 0, & \text{ if }-n_2 < \frac{\alpha}{\sqrt{1-\alpha^2}}n_1\\
1, & \text{ if }-n_2 > \frac{\alpha}{\sqrt{1-\alpha^2}}n_1
\end{cases}.
$$
And for the $\beta$-model (with $\hat{\A}$):
$$
x^u_1 =  \begin{cases} 
0, & \text{ if }n_1 > 0\\
1, & \text{ if }n_1 < 0
\end{cases}, \quad 
x^u_2 =  \begin{cases} 
0, & \text{ if }-n_2 < \frac{\beta}{\sqrt{1-\beta^2}}n_1\\
1, & \text{ if }-n_2 > \frac{\beta}{\sqrt{1-\beta^2}}n_1
\end{cases}.
$$
Note that due to the construction both models agree on the value of the binary variable $x^u_1$.

\begin{figure*}
    \centering
    \includegraphics[scale=0.75]{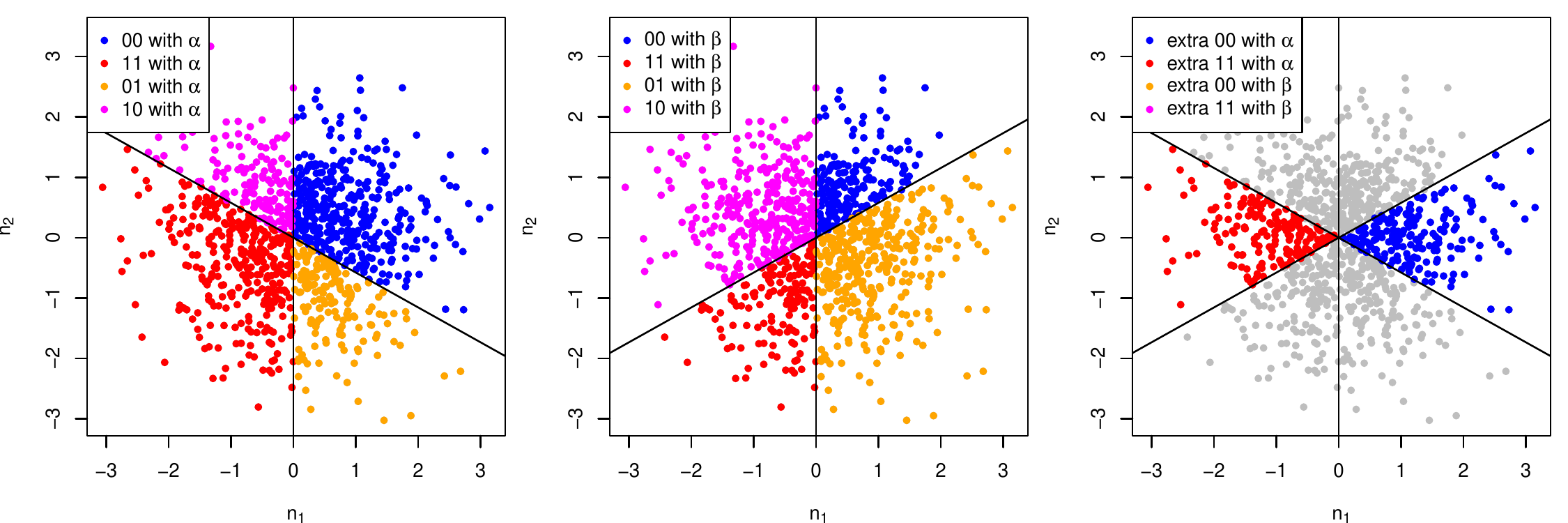}
\caption{Bivariate standard normal $\n$ and colors indicating which binary assignments are implied with $\alpha=0.5$ (left) and with $\beta=-0.5$ (center). For this case with zero means, with higher correlation value $\alpha$ we get more $00$ and $11$ assignments as can be seen from the rightmost plot. Grey points in the rightmost plot do not imply extra 00 or 11 assignments with either correlation value and are irrelevant for the proof. 
\label{fig:thm_plot1} }
\end{figure*}

With $\beta$ we get extra assignments such that $x^u_1=x^u_2=0$ if:
\begin{eqnarray}
n_1 &>& 0 \quad \text{AND} \quad 
-n_2 \in \left[\frac{\alpha}{\sqrt{1-\alpha^2}}n_1,\frac{\beta}{\sqrt{1-\beta^2}}n_1 \right] \label{eq:full1zero}
\end{eqnarray} 
Since  $\alpha > \beta$ and $x/\sqrt{1-x^2}$ is increasing, the interval for $n_2$ is empty, and no $\n$ implies $x^u_1=x^u_2=0$ with $\beta$ if not with $\alpha$. Suppose $\n$ is such that
\begin{eqnarray*}
n_1 &>& 0
\quad  \text{ AND } \quad
-n_2 \in \left[\frac{\beta}{\sqrt{1-\beta^2}}n_1, \frac{\alpha}{\sqrt{1-\alpha^2}}n_1 \right].
\end{eqnarray*}
The binary values implied are $x^u_1=x^u_2=0$ with $\alpha$ and $x^u_1=0,x^u_2=1$ with $\beta$.
Since  $\alpha > \beta$ and $x/\sqrt{1-x^2}$ is increasing, the interval for $n_2$ has non-zero measure.
Thus there is a nonzero measure for obtaining extra $x^u_1=x^u_2=0$ with $\alpha$. See Figure~\ref{fig:thm_plot1} for pictorial representation of the situation when $\alpha=0.5$, $\beta=-0.5$. 
\end{proof}

\begin{proof}
We can focus here on bivariate models as the multivariate normal for $\q^u$ can be straightforwardly marginalized to the bivariate case.
Suppose the two models respectively imply:
\begin{equation}
\q^u \sim \mathcal{N}( \mub_{ \q}^u , \Sigmab_{ \q}^u),\quad \hat{\q}^u \sim \mathcal{N}( \hat{\mub}_{ \q}^u , \hat{\Sigmab}_{ \q}^u),\label{q_eq2}\end{equation}
Then the marginals are:
\begin{eqnarray*}
P(x_1^u=1)&=&\Phi(0|\mu_1,\sigma_1^2)=\Phi\left(-\frac{\mu_1}{\sigma_1}|0,1\right),\\
P(\hat{x}_1^u=1)&=&\Phi(0|\hat{\mu}_1,\hat{\sigma}_1^2)=\Phi\left(-\frac{\hat{\mu}_1}{\hat{\sigma}_1}|0,1\right),\end{eqnarray*}
where $\mu_1$, $\hat{\mu}_1$, $\sigma_1$, and $\hat{\sigma}_1$ denote the parameters in Equation~\ref{q_eq2}. For the models to imply the same distributions the marginals need to be the same. The same applies for $x_2^u$ with parameters $\mu_2$, $\hat{\mu}_2$, $\sigma_2$, and $\hat{\sigma}_2$. Since $\Phi$ is monotonically increasing, we can assume from here on:
$$
\mu_1\hat{\sigma}_1 = \hat{\mu}_1\sigma_1, \quad
 \mu_2 \hat{\sigma}_2 =\hat{\mu}_2 \sigma_2.
$$
Due to Equations~10 and~11 in the main paper we can also assume we are dealing with ``standardized'' models where the diagonals of the covariances are units for both models. We get:
$$
\mu_1 = \hat{\mu}_1,\quad \mu_2 =\hat{\mu}_2,\quad
\hat{\sigma}_1 = \sigma_1 =  \hat{\sigma}_2 = \sigma_2 = 1.
$$

The correlation/covariance matrices for $\q$ and $\hat{\q}$ are:
$$
\Sigmab_{ \q}^u=
 \left( \begin{array}{ccc} 
 1 &\alpha \\
\alpha & 1 \end{array}\right),\quad 
\hat{\Sigmab}_{ \q}^u=
 \left( \begin{array}{ccc} 
 1 &\beta \\
\beta & 1 \end{array}\right)
$$
We study the difference in the implied binary distribution by the two models by creating the Gaussian distributions for $\q^u$ and $\hat{\q}^u$ from a single standard multivariate Gaussian source. The distributions can be formed from a standard normal $\mathbf{n}\sim \mathcal{N}(\mathbf{0}, \mathbf{I})$, for example by multiplying with matrices$$
\A= \left( \begin{array}{ccc} 
 1 & 0 \\
\alpha
& \sqrt{1-\alpha^2}\end{array}\right), \quad \hat{\A}= \left( \begin{array}{ccc} 
 1 & 0 \\
\beta
& \sqrt{1-\beta^2}\end{array}\right)
$$
such that 
$$
\q= \A \mathbf{n}+\mub, \quad \hat{\q} = \hat{\A} \mathbf{n}+\hat{\mub},
$$
where $\mub=\hat{\mub}$ due to the earlier. We will assume $\alpha > \beta$ without loss of generality. Let's look at which values for $\n$ result in different assignments for the binary variables. Recall that the assignment is determined deterministically by the quadrant $\q^u$ and $\hat{\q}^u$ land in. Intuitively, the model with higher correlation $\alpha$ implies more similar values for the binary variables. For the $\alpha$ model:
$$
x_1^u =  \begin{cases} 0, & \text{ if }n_1 > -\mu_1 \\
1, & \text{ if }n_1 < -\mu_1
\end{cases} , \quad x_2^u = \begin{cases} 0, & \text{ if }-n_2 < \frac{\alpha}{\sqrt{1-\alpha^2}}n_1+\frac{1}{\sqrt{1-\alpha^2}}\mu_2\\
1, & \text{ if }-n_2 > \frac{\alpha}{\sqrt{1-\alpha^2}}n_1+\frac{1}{\sqrt{1-\alpha^2}}\mu_2
\end{cases}.
$$
And for the $\beta$ model:
$$
\hat{x}_1^u =  \begin{cases} 
0, & \text{ if }n_1 > -\mu_1 \\
1, & \text{ if }n_1 < -\mu_1
\end{cases}, \quad 
\hat{x}_2^u =  \begin{cases} 
0, & \text{ if }-n_2 < \frac{\beta}{\sqrt{1-\beta^2}}n_1+\frac{1}{\sqrt{1-\beta^2}}\mu_2\\
1, & \text{ if }-n_2 > \frac{\beta}{\sqrt{1-\beta^2}}n_1+\frac{1}{\sqrt{1-\beta^2}}\mu_2
\end{cases}.
$$
Due to the construction both models agree on the value of the binary variable $x_1^u$.

In the zero-mean case presented above, we got more 00 \emph{and} 11 assignments with the higher correlation $\alpha$ than with the lower correlation $\beta$ (Figure~1). Here we can only prove that we always get more 00 \emph{or} 11 assignments, since changing the mean complicates matters (Figures~2 and~3). This is still enough for showing that the distributions are different. First, we show that the lower correlation $\beta$ cannot give extra 00 \emph{and} 11 assignments in comparison to $\alpha$ (separately for positive and negative $\alpha$).

\paragraph{Case $\alpha > 0$} With $\beta$ we get additional assignments such that $x_1^u=x_2^u=0$ if:
\begin{eqnarray}
n_1 > -\mu_1 \quad \text{AND} \quad 
-n_2 \in \left[\frac{\alpha}{\sqrt{1-\alpha^2}}n_1+\frac{1}{\sqrt{1-\alpha^2}}\mu_2,\frac{\beta}{\sqrt{1-\beta^2}}n_1+\frac{1}{\sqrt{1-\beta^2}}\mu_2 \right] \label{eq:full1} 
\end{eqnarray} 
Replacing $n_1$ with smaller $-\mu_1$ in the lower bound gives a necessary condition for this:
\begin{eqnarray}
-n_2 
 \in \left[-\frac{\alpha}{\sqrt{1-\alpha^2}}\mu_1+\frac{1}{\sqrt{1-\alpha^2}}\mu_2, \frac{\beta}{\sqrt{1-\beta^2}}n_1+\frac{1}{\sqrt{1-\beta^2}}\mu_2 \right] \label{eq:nec1}
\end{eqnarray}
With $\beta$ we get additional assignments $x^u_1=x^u_2=1$ if:
\begin{eqnarray}
n_1 < -\mu_1 \quad \text{AND} \quad
-n_2 \in \left[\frac{\beta}{\sqrt{1-\beta^2}}n_1+\frac{1}{\sqrt{1-\beta^2}}\mu_2,\frac{\alpha}{\sqrt{1-\alpha^2}}n_1+\frac{1}{\sqrt{1-\alpha^2}}\mu_2\right] \label{eq:full2}
\end{eqnarray}
Replacing $n_1$ with larger $-\mu_1$ in the upper bound gives a necessary condition:
\begin{eqnarray}
-n_2 \in \left[\frac{\beta}{\sqrt{1-\beta^2}}n_1+\frac{1}{\sqrt{1-\beta^2}}\mu_2,-\frac{\alpha}{\sqrt{1-\alpha^2}}\mu_1+\frac{1}{\sqrt{1-\alpha^2}}\mu_2\right]\label{eq:nec2}
\end{eqnarray}
Since the lower bound of  Equation~\ref{eq:nec1} matches the upper bound of Equation~\ref{eq:nec2}, and the bound is constant with respect to $\n$, both necessary conditions cannot be fulfilled given any fixed model. Therefore, the conditions the latter were necessary to, Equation~\ref{eq:full1} and Equation~\ref{eq:full2} respectively, will not be satisfied either for any fixed model. Note that either Equation~\ref{eq:full1} or Equation~\ref{eq:full2} can be satisfied alone.

\paragraph{Case $\alpha < 0$} Also $\beta <0$ here. With $\beta$ we get additional assignments such that $x^u_1=x^u_2=0$ if:
\begin{eqnarray}
n_1 > -\mu_1 \quad \text{AND} \quad -n_2
\in \left[\frac{\alpha}{\sqrt{1-\alpha^2}}n_1+\frac{1}{\sqrt{1-\alpha^2}}\mu_2,\frac{\beta}{\sqrt{1-\beta^2}}n_1+\frac{1}{\sqrt{1-\beta^2}}\mu_2 \right] \label{eq:full1neg}
\end{eqnarray} 
Replacing $\beta n_1$ with larger $-\beta \mu_1$ in the upper bound gives a necessary condition for this is:
\begin{eqnarray}
-n_2
\in \left[\frac{\alpha}{\sqrt{1-\alpha^2}}n_1+\frac{1}{\sqrt{1-\alpha^2}}\mu_2,-\frac{\beta}{\sqrt{1-\beta^2}}\mu_1+\frac{1}{\sqrt{1-\beta^2}}\mu_2 \right] \label{eq:nec1neg}
\end{eqnarray}
With $\beta$ we get additional assignments $x^u_1=x^u_2=1$ if:
\begin{eqnarray}
n_1 < -\mu_1 \quad \text{AND} \quad
-n_2 \in \left[\frac{\beta}{\sqrt{1-\beta^2}}n_1+\frac{1}{\sqrt{1-\beta^2}}\mu_2,\frac{\alpha}{\sqrt{1-\alpha^2}}n_1+\frac{1}{\sqrt{1-\alpha^2}}\mu_2 \right] \label{eq:full2neg} 
\end{eqnarray}
Replacing $\beta n_1$ with smaller $-\beta \mu_1$ in the lower bound gives a necessary condition:
\begin{eqnarray}
-n_2 \in \left[-\frac{\beta}{\sqrt{1-\beta^2}}\mu_1+\frac{1}{\sqrt{1-\beta^2}}\mu_2,\frac{\alpha}{\sqrt{1-\alpha^2}}n_1+\frac{1}{\sqrt{1-\alpha^2}}\mu_2 \right]\label{eq:nec2neg}
\end{eqnarray}
Since the upper bound of  Equation~\ref{eq:nec1neg} matches the lower bound of Equation~\ref{eq:nec2neg}, and the bound is constant with respect to $\n$, both necessary conditions cannot be fulfilled given any fixed model. Therefore the conditions the previous were respectively necessary to, Equation~\ref{eq:full1neg} and Equation~\ref{eq:full2neg}, will not be satisfied either for any fixed model. Note that either Equation~\ref{eq:full1neg} or Equation~\ref{eq:full2neg} can be satisfied alone.

\begin{figure*}
    \centering
    \includegraphics[scale=0.75]{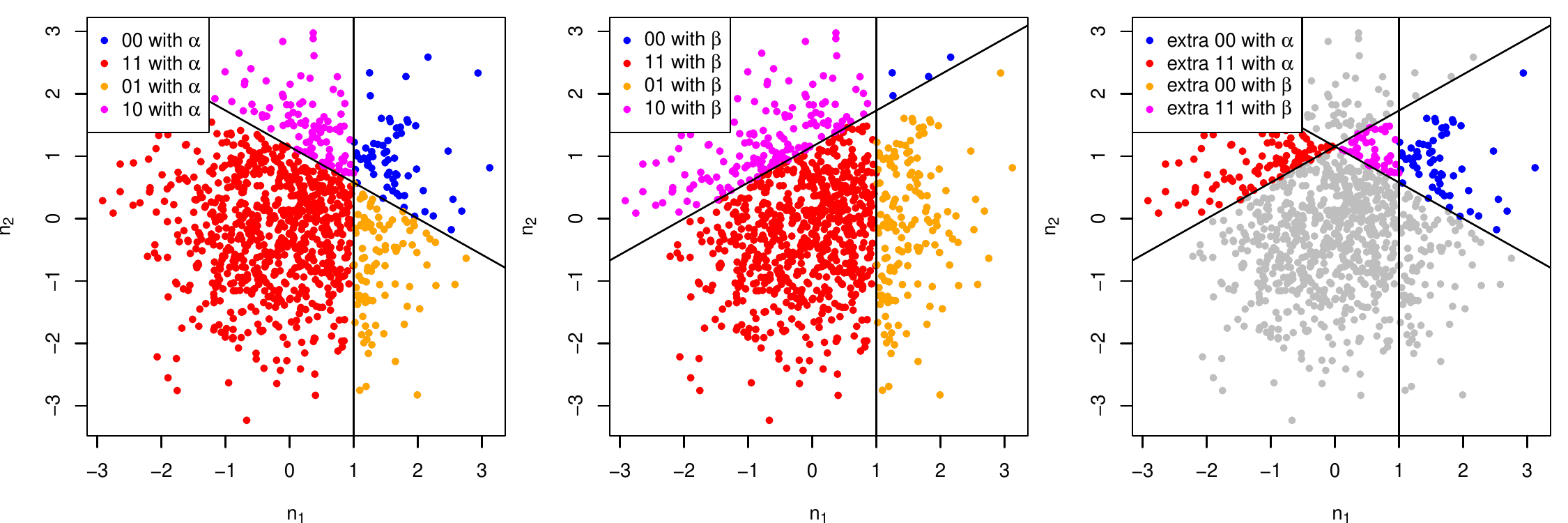}
\caption{Bivariate standard normal $\n$ and colors indicating which binary assignments are implied with $\alpha=0.5$ (left) and with $\beta=-0.5$ (center). For this case with $\mu_1=-1,\mu_2=-1$, with higher correlation value $\alpha$ we (provably) get more $00$ assignments as can be seen from the rightmost plot. Grey points in the rightmost plot do not imply extra 00 or 11 assignments with either correlation value and are irrelevant for the proof. \label{fig:thm_plot3} }
\end{figure*}

\paragraph{Extra 00 with $\alpha$} 
Suppose Equation~\ref{eq:full1} or Equation~\ref{eq:full1neg} is not satisfied. This means that no $\n$ implies $x^u_1=x^u_2=0$ with $\beta$ if not with $\alpha$. Suppose $\n$ is such that
\begin{eqnarray*}
n_1 &>& \max \left(-\mu_1, 
\mu_2 \left(\frac{1}{\sqrt{1-\beta^2}} - \frac{1}{\sqrt{1-\alpha^2}}\right) \big/ \left(\frac{\alpha}{\sqrt{1-\alpha^2}}-\frac{\beta}{\sqrt{1-\beta^2}}\right) 
\right)  \text{ and } \\
-n_2 &\in& \left[\frac{\beta}{\sqrt{1-\beta^2}}n_1+\frac{1}{\sqrt{1-\beta^2}}\mu_2, \frac{\alpha}{\sqrt{1-\alpha^2}}n_1+\frac{1}{\sqrt{1-\alpha^2}}\mu_2 \right].
\end{eqnarray*}
The binary values implied are $x^u_1=x^u_2=0$ with $\alpha$ and $x^u_1=0,x^u_2=1$ with $\beta$.
Furthermore, the following shows that interval for $-n_2$ has non-zero measure. The first multiplication is permitted as the $x/\sqrt{1-x^2}$ is increasing and $\alpha > \beta$.
\begin{eqnarray*}
n_1 &>&
\mu_2\left(\frac{1}{\sqrt{1-\beta^2}} - \frac{1}{\sqrt{1-\alpha^2}}\right)/\left(\frac{\alpha}{\sqrt{1-\alpha^2}}-\frac{\beta}{\sqrt{1-\beta^2}}\right)\quad || \cdot \left(\frac{\alpha}{\sqrt{1-\alpha^2}}-\frac{\beta}{\sqrt{1-\beta^2}}\right) \\
\left(\frac{\alpha}{\sqrt{1-\alpha^2}}-\frac{\beta}{\sqrt{1-\beta^2}}\right) n_1 &>&
\mu_2\left(\frac{1}{\sqrt{1-\beta^2}} - \frac{1}{\sqrt{1-\alpha^2}}\right)\\
\frac{\alpha}{\sqrt{1-\alpha^2}}n_1  &>&
\frac{\beta}{\sqrt{1-\beta^2}}n_1+
\mu_2\left(\frac{1}{\sqrt{1-\beta^2}} - \frac{1}{\sqrt{1-\alpha^2}}\right)\\
\frac{\alpha}{\sqrt{1-\alpha^2}}n_1+\frac{1}{\sqrt{1-\alpha^2}}\mu_2
&>&
\frac{\beta}{\sqrt{1-\beta^2}}n_1+
\mu_2\left(\frac{1}{\sqrt{1-\beta^2}} - \frac{1}{\sqrt{1-\alpha^2}}\right)+\frac{1}{\sqrt{1-\alpha^2}}\mu_2 \\
&=& \frac{\beta}{\sqrt{1-\beta^2}}n_1+\frac{1}{\sqrt{1-\beta^2}}\mu_2.
\end{eqnarray*}
Thus there is a nonzero measure for obtaining extra $x^u_1=x^u_2=0$ with $\alpha$. See Figure~\ref{fig:thm_plot3} for pictorial representation of the situation when $\alpha=0.5$, $\beta=-0.5$, $\mu_1=-1,\mu_2=-1$.

\begin{figure*}
    \centering
    \includegraphics[scale=0.75]{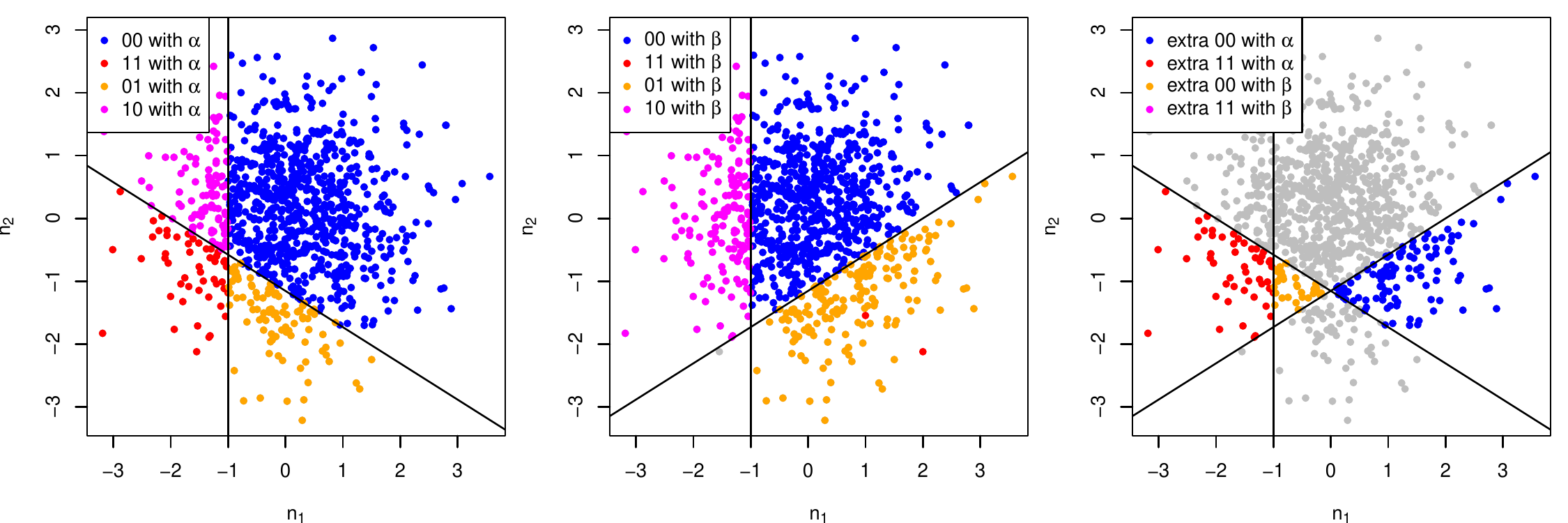}
\caption{Bivariate standard normal $\n$ and colors indicating which binary assignments are implied with $\alpha=0.5$ (left) and with $\beta=-0.5$ (center). For this case with $\mu_1=1,\mu_2=1$, with higher correlation value $\alpha$ we (provably) get more $11$ assignments as can be seen from the rightmost plot. Grey points in the rightmost plot do not imply extra 00 or 11 assignments with either correlation value and are irrelevant for the proof.  \label{fig:thm_plot2} }
\end{figure*}

\paragraph{Extra 11 with $\alpha$} 
Suppose Equation~\ref{eq:full2} or Equation~\ref{eq:full2neg} is not satisfied. This means that no $\n$ implies $x^u_1=x^u_2=1$ with $\beta$ if not with $\alpha$.
Suppose $\n$ is such that
\begin{eqnarray*}
n_1 &<& \min \left(-\mu_1,  
\mu_2(\frac{1}{\sqrt{1-\beta^2}} - \frac{1}{\sqrt{1-\alpha^2}})/(\frac{\alpha}{\sqrt{1-\alpha^2}}-\frac{\beta}{\sqrt{1-\beta^2}}) 
\right)  \text{ and }\\
-n_2 &\in& \left[ \frac{\alpha}{\sqrt{1-\alpha^2}}n_1+\frac{1}{\sqrt{1-\alpha^2}}\mu_2, \frac{\beta}{\sqrt{1-\beta^2}}n_1+\frac{1}{\sqrt{1-\beta^2}}\mu_2 \right].
\end{eqnarray*}
 The binary values implied are $x^u_1=x^u_2=1$ with $\alpha$ and $x^u_1=1,x^u_2=0$ with $\beta$.
Furthermore, the following shows that interval for $-n_2$ has non-zero measure. The first multiplication is permitted as the $x/\sqrt{1-x^2}$ is increasing and $\alpha > \beta$.
\begin{eqnarray*}
n_1 &<& 
\mu_2\left(\frac{1}{\sqrt{1-\beta^2}} - \frac{1}{\sqrt{1-\alpha^2}}\right) \big/ \left(\frac{\alpha}{\sqrt{1-\alpha^2}}-\frac{\beta}{\sqrt{1-\beta^2}}\right) \quad || \cdot \left(\frac{\alpha}{\sqrt{1-\alpha^2}}-\frac{\beta}{\sqrt{1-\beta^2}}\right) \\
\left(\frac{\alpha}{\sqrt{1-\alpha^2}}-\frac{\beta}{\sqrt{1-\beta^2}}\right)n_1 &<& 
\mu_2\left(\frac{1}{\sqrt{1-\beta^2}} - \frac{1}{\sqrt{1-\alpha^2}}\right)  \\
\frac{\alpha}{\sqrt{1-\alpha^2}}n_1 &<&\frac{\beta}{\sqrt{1-\beta^2}}n_1+
\mu_2 \left(\frac{1}{\sqrt{1-\beta^2}} - \frac{1}{\sqrt{1-\alpha^2}}\right)  \\
\frac{\alpha}{\sqrt{1-\alpha^2}}n_1 + \frac{1}{\sqrt{1-\alpha^2}}\mu_2
&<&
\frac{\beta}{\sqrt{1-\beta^2}}n_1+
\mu_2(\frac{1}{\sqrt{1-\beta^2}} - \frac{1}{\sqrt{1-\alpha^2}})
+\frac{1}{\sqrt{1-\alpha^2}}\mu_2 \\
&=& \frac{\beta}{\sqrt{1-\beta^2}}n_1+\frac{1}{\sqrt{1-\beta^2}}\mu_2.
\end{eqnarray*}
Thus there is a nonzero measure for obtaining extra $x^u_1=x^u_2=1$ with $\alpha$. 
See Figure~\ref{fig:thm_plot2} for pictorial representation of the situation when $\alpha=0.5$, $\beta=-0.5$, $\mu_1=1,\mu_2=1$.
\end{proof}
\section{Proof of Theorem 3}

\begin{theorem}
If two models
$\mathcal{M}$
and $\hat{\mathcal{M}}$ with $n=n_z$
imply the same correlation matrices for $\q^u$ (in a given segment)
 then the means $\mub_\z^u$ can be adjusted such that the implied binary distributions are identical. 
\end{theorem}
\begin{proof}
If the models imply sample correlations for $\q^u$ they satisfy Equation 11.
Thus determine the positive diagonal matrices $\Qb^u$ from Equation~11  in the main paper, from the diagonal. Then solve for $\mub_\z^u$ from Equation~10 in the main paper since $\A$ and $\Qb^u$ are invertible. 
Since the equations are satisfied, the implied binary distributions are identical.
\end{proof}

\section{Evaluation: Mean Cosine Similarity} \label{sec:mcs}

In the binary case, it is more relevant to evaluate the estimated \textbf{mixing matrix} than the sources, since the binarization process adds much more noise than simply adding Gaussian noise to the observations.
For this purpose, a similar procedure to mean correlation coefficent (MCC) is applied between the estimated mixing matrix and the true mixing matrix.

When there are only two components, the mixing matrix $\mathbf{A} \in \mathbb{R}^{2 \times 2}$ can be written considering its column vectors $\mathbf{A} = [\mathbf{a}_1, \mathbf{a}_2]$.
Each vector contains only two elements, so the correlation coefficient cannot be used, since $r(\mathbf{v}_1, \mathbf{v}_2) = 1 \quad \forall \ \mathbf{v}_1, \mathbf{v}_2 \in \mathbb{R}^2$. 
In addition, even if $n > 2$, the MCC is undesired because by subtracting the means of each vector, the correlation between ``shifted'' vectors is the same as if they were not shifted: $r(\mathbf{v}_1 + \mathbf{d}, \mathbf{v}_2) = r(\mathbf{v}_1, \mathbf{v}_2)$ for any $\mathbf{d} \in \mathbb{R}^2$.

Therefore, we employ the \textbf{Mean Cosine Similarity (MCS)} instead of the MCC. The MCS uses the cosine similarity -- instead of the correlation coefficient -- to determine whether the vectors of the true and estimated matrices are aligned:
\begin{equation}
    \cos( \mathbf{a}_1, \mathbf{a}_2)=
    \frac{\mathbf{a}_1 \cdot \mathbf{a}_2}{\|\mathbf{a}_1\|\|\mathbf{a}_2\|}
\end{equation}

Let us denote the $i^{\text{th}}$ column of a matrix $\mathbf{A} \in \mathbb{R}^{n \times n_s}$ as $\A[,i]$. In the MCS calculation, we aim to compare each column of $\mathbf{A}$ with each column of the estimated matrix $\hat{\mathbf{A}}$, thus getting a pair-wise cosine similarity. For simplicity, we consider a column permutation $p$ of matrix $\hat{\mathbf{A}}$ as $\hat{\mathbf{A}}[,p[i]]$. We compute the mean cosine similarity across all the columns for each permutation, and take the maximum, hence defining the MCS as:
\begin{equation}
    \text{MCS}(\mathbf{A}, \hat{\mathbf{A}}) = \max_p \left( \dfrac{1}{n_s} \sum_{i=1}^{n_s} \mid \cos(\mathbf{A}[,i], \hat{\mathbf{A}}[,p[i]]) \mid \right).
\end{equation}
Instead of actually going through the permutation, the computation can be efficiently performed via a linear assignment problem or a linear program.

\section{Variational Autoencoder for Binary Data (\texttt{linear iVAE})} 
\label{appendix_ivae}

\paragraph{Estimation}
The variational autoencoder\footnote{The notation here differs slightly from the previous in order to follow the notation in [Khemakhem \textit{et al.}, 2019] more closely.} iVAE
[Khemakhem \textit{et al.}, 2019] aims to estimate the observed data distribution $p(\mathbf{x}|\mathbf{u})=\int p(\mathbf{x}|\mathbf{z})p(\mathbf{z}|\mathbf{u})d\mathbf{z}$.
Given a dataset $\mathcal{D} = \{ (\mathbf{x}_i, \mathbf{u}_i) \}_{i}$, let $q_\mathcal{D}(\mathbf{x},\mathbf{u})$ be the empirical data distribution. The model learns by maximizing a lower bound $\mathcal{L}$ of the data log-likelihood 
\begin{equation}
    \mathbb{E}_{q_\mathcal{D}(\mathbf{x},\mathbf{u})}[ \log p_{\boldsymbol{\theta}} (\mathbf{x}|\mathbf{u})] \geq \mathcal{L}(\boldsymbol{\theta},\boldsymbol{\phi}).
\end{equation}
The loss function is: 
\begin{equation}
\begin{aligned}
\label{eq:ivae_loss}
    \mathcal{L}(\boldsymbol{\theta},\boldsymbol{\phi}) &:= \mathbb{E}_{q_\mathcal{D}(\mathbf{x},\mathbf{u})}[\mathbb{E}_{q_{\boldsymbol{\phi}}(\mathbf{z}|\mathbf{x},\mathbf{u})}[\log p_{\boldsymbol{\theta}}(\mathbf{x}, \mathbf{z}|\mathbf{u}) - \log q_{\boldsymbol{\phi}} (\mathbf{z|x,u})]] \\
    &= \mathbb{E}_{q_\mathcal{D}(\mathbf{x},\mathbf{u})}[\mathbb{E}_{q_{\boldsymbol{\phi}}(\mathbf{z}|\mathbf{x},\mathbf{u})}[\log p_{\boldsymbol{\theta}}(\mathbf{x} | \mathbf{z}, \mathbf{u})] + \mathbb{E}_{q_{\boldsymbol{\phi}}(\mathbf{z}|\mathbf{x},\mathbf{u})}[\log p_{\boldsymbol{\theta}}(\mathbf{z}|\mathbf{u})]
    -\mathbb{E}_{q_{\boldsymbol{\phi}}(\mathbf{z}|\mathbf{x},\mathbf{u})}[\log q_{\boldsymbol{\phi}}(\mathbf{z}|\mathbf{x},\mathbf{u})]].
\end{aligned}
\end{equation}

\begin{figure}[t]
    \centering
    \includegraphics[scale=0.8]{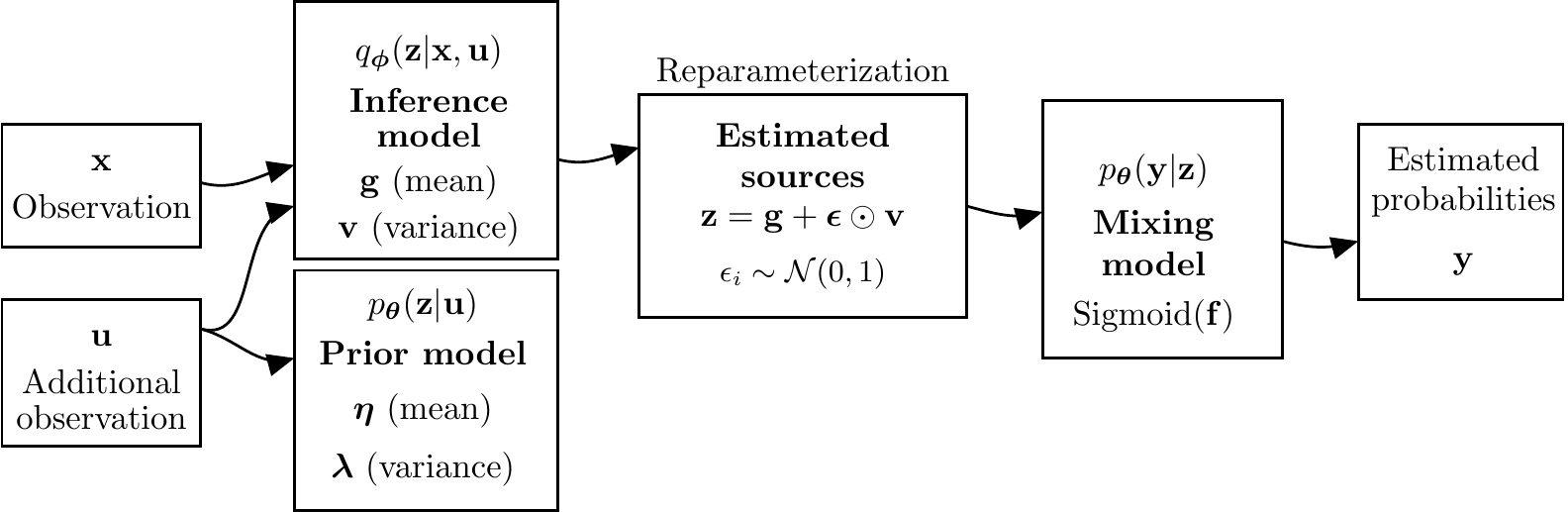}
    \caption{Binary linear iVAE illustration. In VAE terminology: the inference model is equivalent to the encoder, and the mixing model is equivalent to the decoder. The iVAE uses an additionally observed variable $\mathbf{u}$ to estimate the inference model. Additionally, the iVAE estimates a ``prior" model for such additionally observed variables. Different from the continuous iVAE, the mixing model does not model the noise explicitly. Also in contrast to the continuous iVAE, the outputs of the model are the estimated probabilities, not the estimated observations. To obtain the probability of each element being 1, a Sigmoid function is applied element-wise to the output of the mixing model. Variables in bold under the model names denote the transformations learned by the model and are described in detail in the text.}
    \label{fig:binary_ivae}
\end{figure}

To compute the loss function, the expectation over the data distribution is implemented as an average over data samples. In order to deal with expectation over $q_{\boldsymbol{\phi}}(\mathbf{z|x,u})$, we use the reparametrization trick and \emph{draw} vectors $\mathbf{z}$ from $q_{\boldsymbol{\phi}}(\mathbf{z|x,u})$. 

To further develop iVAEs for binary data--which we refer to as \texttt{linear iVAE} in this paper---, we notice that we are working with a factorized Bernoulli observational model. The loss terms developed previously in the continuous iVAE model can remain the same for the inference model and the prior model.
However, the loss term referring to the \textbf{mixing model} should be modified, since the data follows a \textbf{multivariate Bernoulli distribution}. We draw $\z^{(i)} \sim q_{\boldsymbol{\phi}}(\mathbf{z}|\mathbf{x},\mathbf{u})$ using the output of the inference model in the reparameterization trick $\z^{(i)} = \g(\x, \uu) + \mathbf{v}(\x, \uu) \odot \boldsymbol{\epsilon}^{(i)}$. Thus, the loss term relating to the mixing model can be given as:
\begin{equation}
\begin{aligned}
    \mathbb{E}_{q_{\boldsymbol{\phi}}(\mathbf{z}|\mathbf{x},\mathbf{u})}[\log p_{\boldsymbol{\theta}}(\mathbf{x} \vert \mathbf{z}, \mathbf{u})]
    &= \mathbb{E}_{q_{\boldsymbol{\phi}}(\mathbf{z}|\mathbf{x},\mathbf{u})}[\log p_{\boldsymbol{\theta}}(\mathbf{x} \vert \mathbf{z})] \approx \dfrac{1}{l} \sum_{j=1}^l \log p_{\boldsymbol{\theta}}(\mathbf{x} \vert \mathbf{z}^{(i)}) = \dfrac{1}{l} \sum_{i=1}^l \sum_{j=1}^{n} \log p_{\boldsymbol{\theta}}({x}_j | \z^{(i)}) \\
    &= \dfrac{1}{l} \sum_{i=1}^l \sum_{j=1}^{n} \left[x_j \log y_j^{(i)} + (1-x_j) \log (1-y_j^{(i)}) \right] \\
    &= \dfrac{1}{l} \sum_{i=1}^l \sum_{j=1}^{n} \log \text{Bernoulli}(x_j; y_j^{(i)}),
\end{aligned}
\end{equation}
where $y_j$ is the probability of the observation being 1, $0 \leq y_j \leq 1$, and it is modeled by applying an element-wise sigmoid function to the continuous output of the linear mixing model. Notice that $\y^{(i)}$ is a function of the estimated sources $\z^{(i)}$ drawn from the estimated posterior.
Hence, the expectation is approximated by computing the log-probability mass function of a Bernoulli distribution given such probability $y_j$.

\paragraph{Binary model}
In the model defined, all the transformations are linear, and the sources are drawn from a Gaussian distribution given their segment. Compared to the continuous iVAE, which uses nonlinear transformations in all the models, the binary model is linear and introduces changes to the mixing model and to the prior model. The prior model now estimates not only the log-variances but also the means.

When the observed variables are binary, we use a ``Bernoulli MLP'' 
[Kingma and Welling, 2014, Rezende et al., 2014]
as a decoder in the mixing model, which aims to estimate parameters from a Bernoulli distribution instead of a Normal distribution. The mixing model is modified from the continuous case by applying a sigmoid function element-wise to the output of the mixing model. In addition, in the binary case, we do not have an explicit factor accounting for the noise in the mixture, as illustrated in Figure \ref{fig:binary_ivae}.

Following, we describe the model in more detail.
First of all, we notice that for simplicity and numerical stability when modeling the variances in both the inference model and the prior model, the transformations model the log-variances, which can easily be converted to the variances via exponentiation. With this trick, even a linear transformation can suffice for modeling the log-variances, thus making the model simpler.

The \textbf{prior model} is composed of a transformation modeling the prior mean, and a transformation modeling the prior log-variance.
The prior \textbf{mean} is modeled by 
\begin{equation}
    \begin{aligned}
        \boldsymbol{\eta}: \ & \R^m \rightarrow \mathbb{R}^{n_s} \quad
              \ & \mathbf{u} \mapsto \boldsymbol{\eta}(\mathbf{u})
    \end{aligned}
\end{equation}
where $\boldsymbol{\eta}$ is an affine transformation.
So the vector of means is given by $\boldsymbol{\eta} (\mathbf{u}) = \mathbf{W}_{\eta} \mathbf{u} + \mathbf{b}_{\eta}$, with matrix weights $\mathbf{W}_{\eta} \in \mathbb{R}^{n_s \times m}$, and a bias vector $\mathbf{b}_{\eta} \in \R^{n_s}$. 
The prior \textbf{log-variance} is modeled by
\begin{equation}
    \begin{aligned}
        \boldsymbol{\lambda}: \ & \R^m \rightarrow \mathbb{R}^{n_s} \quad
                  & \mathbf{u} \mapsto \boldsymbol{\lambda}(\mathbf{u})
    \end{aligned}
\end{equation}
where $\boldsymbol{\lambda}$ is an affine transformation.
The vector of log-variances is given by $\boldsymbol{\lambda}(\mathbf{u}) = \mathbf{W}_{\lambda} \mathbf{u} + \mathbf{b}_{\lambda}$, in which $\mathbf{W}_{\lambda} \in \mathbb{R}^{n_s \times m}$ are the weights, and $\mathbf{b}_{\lambda} \in \R^{n_s}$ are the biases. Notice that $\boldsymbol{\lambda}$ is unrelated to the notation from the exponential family, since we are modeling both the means and variances.

The \textbf{mixing model} learns a transformation
\begin{equation}
    \begin{aligned}
        \f: \ & \mathbb{R}^{n_s} \rightarrow \mathbb{R}^{n} \quad
             & \z \mapsto \f(\z)
    \end{aligned}
\end{equation}
where $\f$ is a linear transformation resulting in the the continuous output $\f(\mathbf{z}) = \mathbf{W}_f \mathbf{z}$, in which $\mathbf{W}_{f} \in \mathbb{R}^{n \times n_s}$ is the matrix of weights.
Then, the probability of the estimated observed variables is given by
\begin{equation}
    \mathbf{y} = \text{Sigmoid}(\mathbf{W}_{f} \mathbf{z}).
\end{equation}
It is important to notice that each element of $\y$ is an individual probability of the particular observed variable being 1, $\{ y_i = P (x_i = 1) \}_{i=1}^{n}$.

The \textbf{inference model} has a transformation modeling the mean, and a transformation modeling the log-variance of the data. The data \textbf{mean} is modeled by
\begin{equation}
    \begin{aligned}
        \g: \ & \mathbb{R}^{n+m} \rightarrow \mathbb{R}^{n_s}  \quad
             & (\x, \uu) \mapsto \g(\x, \uu)
    \end{aligned}
\end{equation}
where $\g$ is an affine transformation.
We denote the concatenation of the vectors $\x$ and $\uu$ as $\x||\uu$.
The vector of means is given by $\g(\mathbf{x}, \mathbf{u}) = \mathbf{W}_{g} (\mathbf{x}||\mathbf{u}) + \mathbf{b}_g$, for a matrix $\mathbf{W}_{g} \in \mathbb{R}^{n_s \times (n+m)}$, and a bias vector $\mathbf{b}_g \in \R^{n_s}$.
The data \textbf{log-variance} is modeled by
\begin{equation}
    \begin{aligned}
        \vb: \ & \mathbb{R}^{n+m} \rightarrow \mathbb{R}^{n_s} \quad 
             & (\x, \mathbf{u}) \mapsto \vb(\x, \mathbf{u})
    \end{aligned}
\end{equation}
where $\vb$ is an affine transformation.
The vector of log-variances is given by $\vb(\mathbf{x}, \mathbf{u}) = \mathbf{W}_{v} (\mathbf{x}||\mathbf{u}) + \mathbf{b}_v$, where $\W_v \in \R^{n_s \times n+m}$ are the weights and $\mathbf{b}_v \in \R^{n_s}$ the biases.

\section{Further Details}

The experiments were run in computer clusters employing Intel Xeon E5-2680 v4
processors. The running times in Figure~5 (right) in the main paper (as well as all the results in all other experiments) were obtained using a single processor for a specific run. 

\end{document}